%% file: MV.tex
\newlength{\bibsep}
\documentclass[3p, preprint, review, nonatbib, nopreprintline]{elsarticle}
\input{MinhPreamble.tex}
\input{TikzPreamble.tex}

\begin{document}
\sloppy
%
\begin{frontmatter}
\title{Sparse Gaussian-Mixture-Model Q-Functions via\\
Hadamard Overparametrization for Online Reinforcement
Learning}
%
\author{Minh Vu}
\author{Konstantinos Slavakis}
\address{Institute of Science Tokyo, Department of
    Information and Communications Engineering, Yokohama,
    Japan. 
    \\\textnormal{Emails: \texttt{vu.d.a5c3@m.isct.ac.jp,
    slavakis@ict.eng.isct.ac.jp}}
}

\begin{abstract}
    This paper develops an online, off-policy
    policy-iteration framework for reinforcement learning
    (RL), based on sparse Gaussian-mixture-model Q-functions
    (S-GMM-QFs). The framework reconciles streaming,
    non-stationary data with the Riemannian structure of the
    parameter space while handling distributional mismatch
    through experience replay. S-GMM-QFs are introduced via
    Hadamard overparametrization, enabling interpretable
    sparsification through smooth regularization that
    facilitates Riemannian-based optimization.
    Overparametrization allows the framework to adaptively
    identify meaningful components from a large initial
    pool, yielding sparse models where interpretability
    emerges naturally from geometry: each component's
    parameters (means and covariances) explicitly encode its
    geometric role in the ambient state-action space. These
    geometric roles are learned through online gradient
    descent on a smooth objective over a (Cartesian-product)
    Riemannian manifold. Numerical tests demonstrate that
    S-GMM-QFs match or exceed deep RL methods while using
    substantially fewer parameters and achieving faster
    improvement per observed transition. Notably, parameter
    efficiency and interpretability combine to maintain
    strong generalization in low-parameter regimes where
    sparsified deep RL approaches degrade.
\end{abstract}


\end{frontmatter}

\section{Introduction} \label{sec:intro}
Reinforcement learning (RL) is a machine-learning framework in which an agent learns an optimal policy by
interacting with its environment to maximize expected cumulative rewards~\cite{Bertsekas:RLandOC:19,
  Sutton:IntroRL:18}. RL typically models the environment as a Markov decision process (MDP), providing a
rigorous mathematical framework for sequential decision-making problems arising across diverse domain,
such as robotics, wireless communications, data mining, and large language model training.

A key concept in RL is the \textit{Q-function,} which estimates the expected cumulative reward after the
agent takes an action in a given state under a specific policy. Classical approaches like Q-learning
by~\cite{watkins92Qlearning} and SARSA by~\cite{singh00sarsa} use tabular representations of Q-functions,
computing values for all possible state-action pairs. While effective for discrete-space problems, these
methods become impractical for large or continuous state-action spaces. To mitigate this limitation,
significant attention has been devoted to the development of RL algorithms that leverage models
(typically non-linear) to approximate Q-functions.

Approximation models for Q-functions have a long history in
RL. Kernel-based (KB)RL, first introduced
by~\cite{ormoneit02kernel} and the extended variants
in~\cite{ormoneit:autom:02, bae:mlsp:11} models
Q-functions in Banach spaces of bounded functions. Methods
developed via temporal difference (TD)~\cite{sutton88td} such as
\cite{xu07klspi, lagoudakis03lspi, regularizedpi:16}, or
Bellman residual (BR)~\cite{onlineBRloss:16}, and more
recent nonparametric approaches in~\cite{vu23rl, akiyama24proximal, akiyama24nonparametric} represent them in reproducing
kernel Hilbert spaces (RKHSs)~\cite{aronszajn50kernels,
scholkopf2002learning}, thereby exploiting the underlying
geometric structure and computational efficiency afforded by
the reproducing inner product.
A notable drawback of nonparametric (kernel)
approaches, however, is that their models typically expand with the amount of data, 
which can result in significant memory and computational
overhead, particularly with
nonstationary data distributions, therefore limiting
scalability in online settings. Sparsification via
approximate linear independency as in~\cite{xu07klspi,
vu23rl} can alleviate this issue,
but often at the cost of degraded accuracy in the resulting
Q-function estimates. A comprehensive review of KBRL,
least-square (LS)TD, and BR methods, along with their
connections to RKHSs, is provided
by~\cite{akiyama24nonparametric}.

Nonparametric setting could also be found in distributional
RL, a prominent and increasingly influential approach in
RL, such as in~\cite{sato99em, agostini17gmmrl,
choi19distRL, mannor05rl, bellemare17distRL,QRDQN}. In these studies, Q-functions are
usually treated as (statistics of) random variables (RVs).
For instance, \cite{mannor05rl, bellemare17distRL} assumes that samples of
Q-values/functions are jointly Gaussian, an assumption that
when combined with classical least-squares and Gauss-Markov
theory, leads to Kalman-type algorithmic solutions. To
enable more expressive probabilistic modeling,
Gaussian-mixture-models (GMMs)~\cite{Reynolds:GMM:19} have
been widely adopted to approximate either the joint
probability density function (PDF)
$p(Q,\vect{s}, a)$---in which Q-value $Q$, state $\vect{s}$
and action $a$ are all modeled as observations of RVs---or
the conditional $p(Q\given \vect{s}, a)$. This conventional
use of GMMs, which also evolve together with number of
observed data, closely linked to maximum likelihood
estimation, naturally motivates the adoption of
expectation-maximization (EM) procedures, as
in~\cite{sato99em, agostini17gmmrl, mannor05rl}. As such,
distributional RL estimates $Q$ indirectly as a statistical
byproduct of the modeled PDF, usually as the mean of the
conditional $p(Q\given \vect{s}, a)$ in EM-based solutions.


On the other hand, 
usage of deep neural networks as parametric functional
approximators for Q-function, exemplified by
deep Q-networks (DQNs)~\cite{mnih13dqn, hasselt16ddqn, duelingddqn},
provides strong representational power by nature and avoids
the model-growth issues of the aforementioned nonparametric
approaches, albeit at the cost of requiring many complicated
architectures with large number of learnable parameters.
Typically, DQNs
learn from past experience~\cite{lin93experience}, gathered
from previously employed policies and stored in a
\textit{replay buffer},
enabling exploration beyond the current policy. A widely
adopted refinement of this replay mechanism is prioritized
experience replay (PER)~\cite{PER}, which samples
transitions preferentially according to their
temporal-difference (TD) error rather than uniformly, so
that transitions carrying more informative learning signal
are replayed more frequently. While DQNs
offer a practical and powerful means of training RL agents,
they exhibit limitations in several key scenarios. In
particular, deployment in an online setting requires swift
adaptation to new data, which becomes challenging due to the
large number of learnable parameters within the deep
networks. Frequent model update with new data (re-training),
although theoretically possible within simulation scope, is
computationally expensive and requires significant hardware
resources, contradicting lightweight and rapid adaptability
requirements of online learning.
However, these usually large ``black-box'' networks are
vulnerable to sudden changes in dynamic environments and
provide little insight into the features influencing agent
decisions, limiting interpretability of learned models at
hand. 

To reduce the size of deep models, sparsification techniques
are used, following the taxonomy and implementations
benchmarked by~\cite{sparseDRL}. Dense-to-sparse (pruning)
approaches~\cite{dense2sparse} train a dense network as usual
while, over an initial portion of training, progressively
removing its smallest-magnitude connections according to a
fixed schedule until a target sparsity level is reached,
after which the resulting sparse network is trained to
convergence with its connectivity held fixed; because it
starts from and trains a full dense network before
sparsifying, pruning is known to achieve state-of-the-art
performance, but requires computational resources comparable
to (or exceeding) those of ordinary dense training. Sparse
training approaches, by contrast, fix a sparsity pattern from
the start and dynamically adjust connections throughout
training. Sparse evolutionary training
(SET)~\cite{mocanu17scalable} periodically drops a fraction of
the weakest connections and replaces them with an equal
number of new, randomly placed ones, with this drop fraction
annealed via a cosine decay schedule~\cite{cosine}; RigL~\cite{rigl}
follows the same procedure but regrows connections
using the gradient signal rather than at random, so that new
connections are placed where they are expected to reduce the
loss the most. However, even for these sparsified or pruned
models, interpretability remains limited, since none of them
establishes an explicit link between a surviving connection
and the mathematical structure or geometry of the input
space.

Searching for more expressive power in approximate RL,
\cite{Vu:eusipco:25, minh25tsp} introduce Gaussian-mixture-model
Q-functions (GMM-QFs) to represent Q-functions as weighted
sums of multivariate Gaussian components with learnable
weights, mean vectors and covariance matrices.
Unlike distributional RL that uses GMMs to model
PDFs, GMM-QFs are used directly as a parametric
functional approximator of Q-functions. As a result, they
are not constrained by
mixture weights summing to one, nor do they rely on the assumptions
of Gaussianity on observed data, as typically required in
standard distributional RL approaches.
The number of mixture components is user-specified, allowing control over model
complexity and mitigating the curse of dimensionality in
nonparametric approaches. In practice, however, identifying
an appropriate number of components a priori is challenging:
the ideal model complexity is generally unknown and varies
with the learning task, so that an overly small number of
components limits representational capacity, while an overly
large one increases the risk of overfitting and incurs
unnecessary computational overhead. \cite{Vu:eusipco:25, minh25tsp}
employ GMM-QFs within an offline, on-policy policy-iteration (PI) framework via
BR, exploiting Riemannian
geometry of the parameter space.

In offline (batch) RL---the most typical RL scenario---the agent is trained entirely from a static,
pre-collected dataset. In \textit{online RL,} by contrast, the agent learns from sequentially arriving
transitions (states, actions, rewards) while interacting with the environment; training and data
collection occur simultaneously. Online RL is preferable when the environment is non-stationary and
resources are limited---typical conditions in real-world decision-making~\cite{onlineRL22}. This paper
addresses the core challenges of online RL: developing a lightweight, interpretable model that updates
continually as streaming transitions arrive, while handling distributional shift as the policy evolves.

Building upon this foundation, the present work emphasizes
the following key contributions.
\begin{itemize}
\item[\textbf{(C1)}] GMM-QFs are extended into an \textit{online}\/ and \textit{off-policy}\/ PI
  framework that learns from streaming data, while simultaneously constructing an experience buffer.
  Unlike the offline, on-policy setting of~\cite{Vu:eusipco:25, minh25tsp}, this extension reconciles the
  sequential, non-stationary nature of streaming data with the Riemannian-manifold structure of the
  parameter space, and actively counteracts the distributional mismatch introduced by learning off-policy
  from a replay buffer rather than from freshly sampled on-policy data. This buffer, inspired
  by~\cite{lin93experience}, is actively exploited to enhance both exploration and mitigate biased nature
  of on-policy approaches.  Furthermore, several structures of experience buffer are studied to evaluate
  how distribution of experiences affect the decision making process of the proposed Q-functional
  classes.

\item[\textbf{(C2)}] To improve scalability and interpretability of original GMM-QFs in real-world
  systems, a novel class of sparse (S-)GMM-QFs is introduced via Hadamard overparametrization.  Standard
  sparsification approaches penalize mixture weights via non-smooth $\ell_p$-norms ($p \in [0,
    1]$), which are ill-suited to Riemannian optimization. Hadamard overparametrization instead enables
  sparsification through a smooth, differentiable regularizer compatible with Riemannian-based
  optimization while preserving geometric constraints. Overparametrization also enhances representational
  capacity: S-GMM-QFs begin with a large pool of components and adaptively identify only the meaningful
  ones through sparsity, allowing components to evolve in response to task structure rather than being
  fixed \textit{a priori.} The resulting sparse mixtures are inherently interpretable: each surviving
  component is defined by explicit geometric parameters (means and covariances) rather than hidden
  representations, enabling principled analysis of the learned structure without post-hoc attribution
  methods.  Learning proceeds via online gradient descent on a smooth objective over the (product)
  Riemannian manifold. The approach contrasts with Deep RL sparsification, which targets network
  architectures and provides little transparent connection between internal representations and learned
  structure.
  %
\end{itemize}
Numerical experiments on standard benchmark control tasks demonstrate that S-GMM-QFs match or surpass the
performance of DeepRL models using substantially more parameters, with faster improvement per observed
transition, and that this advantage persists in low-parameter-count regimes where sparsified DeepRL
methods degrade substantially.

This manuscript is organized as follows.
\cref{sec:background} reviews the basics of RL and establishes
the notation used throughout the paper. \cref{sec:model}
introduces the original GMM-QFs together with the online policy
iteration (PI) framework. The extended class of S-GMM-QFs,
along with the Hadamard overparametrization, is introduced in
\cref{sec:sparse-gmmq}. The Riemannian optimization approach for
the proposed online PI is detailed in \cref{sec:riemannian},
while the structure of the experience buffer is presented in
\cref{sec:er}. Numerical experiments on standard benchmark
control tasks are reported in \cref{sec:tests}. Finally,
\cref{sec:conclusion} provides concluding remarks and outlines
directions for future research. This manuscript also constitutes
a substantial extension of the short conference paper
\cite{minh26sparse}.

\section{Preliminaries: Background on
RL}\label{sec:background}

Let $\mathfrak{S} \subset \Real^{D_s}$ denote a \textit{continuous}\/ state space, with state vector
$\vect{s} \in \mathfrak{S}$, for some $D_s \in \IntegerPP$ ($\IntegerPP$ denotes the set of all positive
integers). The discrete action space is denoted by $\mathfrak{A}$, with action $a \in \mathfrak{A}$. For
convenience, cardinality of action space $N_a \coloneqq \lvert \mathfrak{A} \rvert < \infty$. An agent at
state $\vect{s} \in \mathfrak{S}$ takes action $a \in \mathfrak{A}$ and transits to a new state
$\vect{s}'\in \mathfrak{S}$ under an unknown transition probability $p(\vect{s}' \given \vect{s}, a)$
with a reward $r(\vect{s}, a)$. The Q-function $Q(\cdot, \cdot) \colon \mathfrak{S} \times \mathfrak{A}
\to \Real \colon (\vect{s}, a) \mapsto Q(\vect{s}, a)$ stands for the long-term cumulative reward
achievable if the agent selects action $a$ in state $\vect{s}$. Following~\cite{Bertsekas:RLandOC:19}, a
(deterministic) policy $\mu(\cdot)$ maps a state to an action, as in $\mu(\cdot)\colon \mathfrak{S} \to
\mathfrak{A} \colon \vect{s} \mapsto \mu(\vect{s})$. Denote also the set of all mappings from
$\mathfrak{S}$ to $\mathfrak{A}$ by $\mathcal{M}$. Let also $\overline{1,N} \coloneqq \{ 1, \ldots, N\}$.

\begin{figure}[t]
    \centering
    \includegraphics[width=.75\linewidth]{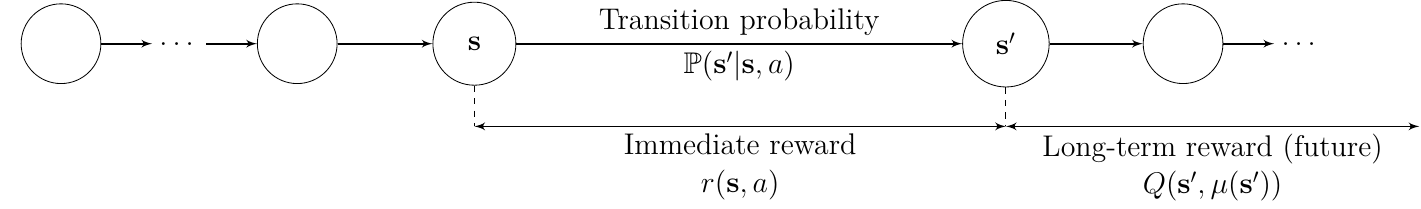}
    \caption{RL as a sequential decision-making process: at
    state $\vect{s}$, the agent takes decision/action $a
    \coloneqq \mu(\vect{s})$, receives a reward
    $r(\vect{s},a)$ and moves to the next state
    $\vect{s}'$, under some transition probability.}
    \label{fig:DP}
\end{figure}

The Q-function is determined by the Bellman
mapping~\cite{Bertsekas:RLandOC:19} through relationship
between immediate rewards and the discounted future values
(see~\cref{fig:DP}).
More precisely, when Q-functions are drawn from the functional space
$\Banach$---typically the Banach space of essentially bounded functions~\cite{Bertsekas:RLandOC:19}---the (classical)
Bellman mapping $\Bellman^{\diamond}_{\mu} \colon \Banach \to \Banach \colon Q \mapsto \Bellman^{\diamond}_{\mu} Q$ for
a policy $\mu(\cdot)$ is defined as: $\forall (\vect{s}, a)$,
\begin{align}
  (\Bellman_{\mu}^{\diamond} Q)(\vect{s}, a)
  & \coloneqq r( \vect{s}, a ) + \alpha \mathbb{E}_{\vect{s}^{\prime} \given
    (\vect{s}, a)} \bigl[ Q(\vect{s}^{\prime}, \mu(\vect{s}^{\prime}))
    \bigr]\,, \label{Bellman.standard.mu}
\end{align}
where $\mathbb{E}_{\vect{s}' \given (\vect{s}, a)} [\cdot]$ is the conditional expectation operator with respect to the
next state $\vect{s}'$ conditioned on $(\vect{s}, a)$ under
the (unknown) transition probability $\mathbb{P}(\vect{s}'
\given \vect{s}, a)$, and $\alpha \in [0, 1)$ being the discount factor. A greedy
  version of \eqref{Bellman.standard.mu} is the Bellman mapping $\Bellman^{\diamond} \colon \Banach \to \Banach \colon Q
  \mapsto (\Bellman^{\diamond} Q)(\vect{s}, a) \coloneqq r( \vect{s}, a ) + \alpha \mathbb{E}_{\vect{s}^{\prime} \given
    (\vect{s}, a)} [ \max_{a^{\prime}\in \mathfrak{A}} Q(\vect{s}^{\prime}, a^{\prime}) ]$.

The fixed-point set of $\Bellman^{\diamond}_{\mu}$ is defined as $\Fix \Bellman^{\diamond}_{\mu} \coloneqq \Set{Q \in
  \Banach \given Q = \Bellman^{\diamond}_{\mu} Q}$. It is well-known that identifying a fixed point $Q^{\diamond}_{\mu}
\in \Fix \Bellman^{\diamond}_{\mu}$ plays a central role in computing optimal policies that maximize cumulative
rewards. Usually, $\alpha \in [0, 1)$ to assure that \eqref{Bellman.standard.mu} becomes a strict contraction, hence
  $\Fix\Bellman^{\diamond}_{\mu}$ is a
  singleton~\cite{Bertsekas:RLandOC:19, hb.plc.book}.

Policy iteration (PI) is a popular framework in
RL~\cite{Bertsekas:RLandOC:19}. PI comprises of two
stages per iteration $n \in \IntegerPP$: \textit{policy
evaluation} and \textit{policy improvement}. In particular,
given a policy $\mu_n$ at iteration $n$, policy evaluation
computes a Q-function $Q_n$ that ``closely'' approximates,
in an appropriate sense, $Q^{\diamond}_{\mu_n} \in
\Fix\Bellman^{\diamond}_{\mu_n}$, while policy improvement
updates the policy according to the following greedy rule:
\begin{equation}\label{eq:policy.impr}
    \mu_{n+1}(\vect{s}) \coloneqq \Argmax_{a \in
    \mathfrak{A}} Q_n(\vect{s}, a)\,, \quad \forall \vect{s}
    \in \mathfrak{S} \,.
\end{equation}
PI iterates this process to generate a sequence of policies
and Q-functions as in $\mu_0 \to Q_0 \to \mu_1 \to Q_1 \to
\dots$, with the goal that the resulting sequence
$(Q_n)_{n=0}^{\infty}$ converges to $Q^{\diamond} \in \Fix
\Bellman^{\diamond}$. The \textit{optimal policy} then is
defined as $\mu^{\diamond}(\vect{s}) \coloneqq \Argmax_{a
\in \mathfrak{A}} Q^{\diamond}(\vect{s}, a)\,, \forall
\vect{s} \in \mathfrak{S}$. 
In the online setting studied in this paper, this iteration
index $n$ coincides with the discrete time step at which a
new transition of streaming data becomes available to the
agent, so that a single PI iteration is carried out per
incoming transition; see~\cref{sec:online-PI} for details.


\section{Gaussian-mixture-model Q-functions for online RL}\label{sec:model}

\subsection{Gaussian-mixture-model Q-functions and parameter
space}\label{subsec:gmmq}
State space in real-world is usually high-dimensional
or/and even continuous, inflicting significant burdens for
tabular methods, where Q-functions must be exactly computed
at \textit{every} $\vect{s} \in \mathfrak{S}$ to identify
the policies---for example, \eqref{eq:policy.impr}. To
address this limitation, functional approximation models for
Q-functions, such as kernel~\cite{ormoneit02kernel,
akiyama24proximal, xu07klspi} and deep neural
networks~\cite{mnih13dqn, PPO} have been attracting research
interests as pivotal alternatives to traditional tabular
methods, introducing \textit{approximate} (A)PI through
their involvement at policy-evaluation stage. A discussion on
prior functional approximation methods for Q-function is
given in~\cite{minh25tsp}.

This paper adopts Gaussian-mixture-model Q-functions
(GMM-QFs), a parametric functional-approximation class introduced
in~\cite{Vu:eusipco:25, minh25tsp}, as a novel \textit{parametric}
model for Q-functions. This formulation provides a
fixed-size model, addressing the growing complexity issue in
nonparametric kernel-based RL.
Extended from the original GMM-QFs in~\cite{minh25tsp}, and for a
user-defined $K\in\IntegerPP$, GMM-QFs are defined as the
following class of functionals:
\begin{align}
  \mathcal{Q}_K \coloneqq \Bigl \{
    Q \colon \mathfrak{S} \times \mathfrak{A} \to \Real \colon (\vect{s}, a) \mapsto Q(\vect{s}, a) 
    \coloneqq \sum_{k=1}^{K} \xi_k(a)
    \mathscr{G}(\vect{s} \given \vect{m}_k, \vect{C}_k)
    \mathop{} \Big \vert\, \xi_k(a) \in \Real, \vect{m}_k \in \Real^{D_s}, \vect{C}_k \in \PD^{D_s}
  \Bigl \} \,, \label{eq:GMMQF}
\end{align}
where the Gaussian component $\mathscr{G}(\vect{s} \given \vect{m}_k, \vect{C}_k) \coloneqq \mathscr{G}_k (\vect{s}) \coloneqq \exp[ -(\vect{s}
  - \vect{m}_k)^{\intercal} \vect{C}_k^{-1} (\vect{s} - \vect{m}_k)]$, and $\PD^{D_s}$ denotes the set of all $D_s
\times D_s$ positive definite matrices, while $\intercal$
stands for vector/matrix transposition.  

This formulation extends the original GMM-QFs, by separating
the model into bins according to the available actions in
$\mathfrak{A}$, allowing a more robust approximation and
liberating it from any specific user-defined state-action
combinations. Similar to~\cite{minh25tsp}, class of GMM-QFs
in~\eqref{eq:GMMQF}
also possesses the \textit{universal approximation property}.
\begin{proposition}\label{prop:dense}
  The union $\cup_{K=1}^{\infty} \mathcal{Q}_K$ is dense in the space of all square-(Lebesgue)-integrable functions on
  $\mathfrak{S} \times \mathfrak{A}$.
\end{proposition}
\begin{proof}
    The proof follows verbatim the proof of~\cite[Theorem
    4(iii)]{minh25tsp}.
\end{proof}

Let now the $D_s \times K$ matrix $\vect{M} \coloneqq [\vect{m}_1, \dots, \vect{m}_K]$. Then the
learnable parameters of~\eqref{eq:GMMQF} can be collected as $\vectgr{\Omega} \coloneqq (\vectgr{\Xi},
\vect{M}, \vect{C}_1, \dots, \vect{C}_K)$, where each $\vectgr{\Xi} \in \Real^{K \times N_a}$ (recall
$N_a = |\mathfrak{A}|$) is defined entry-wise by $[\vectgr{\Xi}]_{k,a} \coloneqq \xi_{k}(a)$. It is shown
in~\cite{minh25tsp} that, each $\vectgr{\Omega}$ specifies a single GMM-QFs in
$\mathcal{Q}_K$. Altogether, the parameter space of GMM-QFs~\eqref{eq:GMMQF} is
\begin{equation} \label{eq:param.space}
    \mathfrak{M}_K \coloneqq \Real^{K \times N_a} \times
    \Real^{D_s \times K} \times (\PD^{D_s})^K \,.
\end{equation}
Interestingly, being the Cartesian product of Euclidean
spaces and manifold of symmetric positive definite matrices,
$\mathfrak{M}_K$ inherits a natural Riemannian manifold
structure~\cite{RobbinSalamon:22, Absil:OptimManifolds:08}.
Consequently, the tangent space at $\vectgr{\Omega}$ becomes 
\begin{equation}\label{eq:param.space.tangent}
    T_{\vectgr{\Omega}} \mathfrak{M}_K = \Real^{K \times
    N_a} \times \Real^{D_s \times K} \times
    T_{\vect{C}_1} \PD^{D_s} \times \ldots \times
    T_{\vect{C}_K} \PD^{D_s} \,,
\end{equation}
where $T_{\vect{C}_k} \PD^{D_s}$ stands for the tangent space
to $\PD^{D_s}$ at $\vect{C}_k$, known to be the set of all
$D_s \times D_s$ symmetric matrices~\cite{RobbinSalamon:22,
Absil:OptimManifolds:08}. Detailed computations on
Riemannian manifold are discussed
in~\cref{subsec:compute.grad}.

\subsection{Online setting of
GMM-QFs}\label{sec:online-PI}

This study focuses on the online-learning setting, in which
streaming data are presented sequentially to the RL agent at
each time instance $n \in \IntegerPP$, aligned with the
PI-iteration index $n$ introduced
in~\cref{sec:background}. At each time instance $n$, the
controlled system provides its current state (data)
$\vect{s}_n$ to the agent, which selects an action $a_n
\coloneqq \mu_n(\vect{s}_n)$ according to the current policy
$\mu_n$. The environment returns a reward (feedback) $r_n$ in
response to this action, and the system transitions to the
next state $\vect{s}_n^{\prime}$, which becomes
$\vect{s}_{n+1} \coloneqq \vect{s}_n^{\prime}$ for the
subsequent time instance. This interaction is summarized by
the tuple $(\vect{s}_n, a_n, r_n, \vect{s}_{n+1})$, which the
agent accumulates in a buffer $\mathcal{B}_n$, of
user-defined capacity $B \in \IntegerPP$, updated according
to $\mathcal{B}_{n+1} = \mathcal{B}_n \cup
\Set{(\vect{s}_n, a_n, r_n, \vect{s}_{n+1})}$. When the
buffer exceeds its capacity, i.e., $\lvert \mathcal{B}_{n+1}
\rvert > B$, its oldest tuple is discarded.

Motivated by the significance of the fixed point
$Q^{\diamond}_{\mu_n} \in \Fix \Bellman^{\diamond}_{\mu_n}$,
the widely used Bellman-residual (BR)
approach~\cite{onlineBRloss:16, regularizedpi:16} estimates
$Q^{\diamond}_{\mu_n}$ by any minimizer of an empirical loss.
In the online setting, this loss is recomputed at every time
step $n$: a dataset $\mathcal{D}_n \coloneqq
\Set{(\vect{s}_t^{(n)}, a_t^{(n)}, r_t^{(n)},
\vect{s}_t'^{(n)})}_{t=1}^{T}$ is sampled from the experience
buffer $\mathcal{B}_n$ (\cref{sec:er}), compensating for
the typical inaccessibility of $\mathbb{E}_{\vect{s}' \given
(\vect{s}, a)} [\cdot]$ in \eqref{Bellman.standard.mu}.
Inspired also by the classical temporal-difference (TD)
strategy~\cite{sutton88td}, the resulting BR objective is,
$\forall \vectgr{\Omega} \in \mathfrak{M}_K$,
\begin{equation} \label{eq:BRM.param}
  L_{\mu_n} (\vectgr{\Omega}; Q_n, \mathcal{D}_n)
  \coloneqq \frac{1}{T} \sum_{t=1}^{T} \Big[
    \sum_{k=1}^{K} \xi_{k} (a_t^{(n)}) \mathscr{G}_k(\vect{s}_t^{(n)})
    - r_t^{(n)} - \alpha
    Q_n (\vect{s}_t'^{(n)}, \mu_n (\vect{s}_t'^{(n)})) \Big]^2 \,.
\end{equation}
Note that only the column vectors of $\vectgr{\Xi}$
corresponding to $\Set{a_t^{(n)}}_{t=1}^{T}$ enter
\eqref{eq:BRM.param} through $\mathcal{D}_n$; the remaining
columns make no contribution, via the associated partial
derivatives, to the Riemannian gradient of $L_{\mu_n}(\cdot ;
Q_n, \mathcal{D}_n)$ (\cref{subsec:compute.grad}).

\begin{algorithm}[t!]
\caption{Online approximate PI via S-GMM-QFs}\label{algo:gmmq}
\SetAlgoLined

\Require{Buffer capacity $B \in \IntegerPP$, sampling
    capacity $T \in \IntegerPP$}

Arbitrarily initialize
    $\vectgr{\Omega}_0 \in \mathfrak{M}_K$, thus
    by~\eqref{eq:GMMQF} $Q_0
    \in \mathcal{Q}_K$, 
    policy $\mu_0 \in \mathcal{M}$, and experience buffer $\mathcal{B}_0 \gets \emptyset$\;
$n\gets 0$, $\vectgr{\Pi}_0 \gets \vect{0}$, $\sigma_0 \gets
    0$ \;
Environment starts at an initial state $\vect{s}_0 \in \mathfrak{S}$\;

\While{$n \in \IntegerP$ \label{line:while.loop}}{%
    Tuple $(\vect{s}_n, a_n, r_n, \vect{s}_{n+1})$ becomes available to the agent\;
    $\mathcal{B}_{n+1} \gets \mathcal{B}_n \cup \{(\vect{s}_n, a_n, r_n, \vect{s}_{n+1})\}$\;
    \lIf{$\lvert \mathcal{B}_{n+1} \rvert > B$}{discard oldest tuple}

    Sample dataset $\mathcal{D}_n$ from $\mathcal{B}_{n+1}$
    following~\cref{sec:er}\;
    Define $\mathscr{L}_{\mu_n}(\cdot; Q_n, \mathcal{D}_n)$ by~\eqref{eq:reg.BRM}\;

    \textbf{Policy evaluation:}

    \label{line:policy.eval} Compute $\grad
    \mathscr{L}_{\mu_n}(\vectgr{\Omega}_n; Q_n,
    \mathcal{D}_n)$ by~\cref{prop:gradients}\label{line:compute.gradient}\;
    Update $\vectgr{\Omega}_{n+1}$ by \eqref{eq:update.Omega}\;
    Define $Q_{n+1}\in\mathcal{Q}_K$ by
    $\vectgr{\Omega}_{n+1}$ via~\eqref{eq:GMMQF};

    \textbf{Policy improvement:} $\mu_{n+1}(\vect{s}) \gets \arg\max_{a \in \mathfrak{A}} Q_{n+1}(\vect{s}, a)$\;

    \lIf{task terminated}{reset $\vect{s}_{n+1} \gets \vect{s}_0$}

    $n \gets n+1$ and go to line~\ref{line:while.loop}\;
}

\end{algorithm}

The novel S-GMM-QFs are integrated into a classical online
approximate policy-iteration (PI) scheme---see
\cref{algo:gmmq}, consisting of the usual two steps: policy
evaluation and policy improvement. In the policy-evaluation
step, the current policy $\mu_n$ guides the update of the
S-GMM-QF estimate via the regularized loss
$\mathscr{L}_{\mu_n}(\vectgr{\Omega}; Q_n, \mathcal{D}_n)$,
which augments~\eqref{eq:BRM.param} with a
sparsity-promoting regularizer introduced
in~\cref{sec:sparse-gmmq}. Since $\mathscr{L}_{\mu_n}$ is
defined on the Riemannian manifold $\mathfrak{M}^{(J)}_K$,
its Riemannian gradient drives the update of
$\vectgr{\Omega}_{n+1}$ via a Riemannian extension of
Adam~\cite{RAdam}, which additionally maintains a momentum
term $\vectgr{\Pi}_n$ and a variance estimate $\sigma_n$; the
full update rule~\eqref{eq:update.Omega} is detailed
in~\cref{sec:riemannian}. In the policy-improvement step, the
agent updates its policy based on the updated S-GMM-QF
estimate.

The policy-improvement step of~\cref{algo:gmmq} operates in a
discrete, finite action space $\mathfrak{A}$, consistent with
the standard PI literature~\cite{Bertsekas:RLandOC:19}. While
this simplifies the update to a direct maximization over
$\mathfrak{A}$, the proposed framework is naturally amenable
to continuous-action settings: an actor-critic
approach~\cite{konda99ac} could employ a separate policy
approximator alongside the S-GMM-QF critic, retaining the
Riemannian policy-evaluation machinery developed here.
Developing such extensions is a promising direction currently
being pursued (see also~\cref{sec:conclusion}).

\section{Sparse GMM-QFs via Hadamard overparametrization}
\label{sec:sparse-gmmq}
    In practice, identifying an appropriate number of
    Gaussian components for GMM-QFs~\eqref{eq:GMMQF} is
    challenging, due to the differences in nature of the
    learning task. This section develops a principled
    approach to
    identify and retain only the most impactful Gaussians by
    promoting sparsity for GMM-QFs.

\subsection{Hadamard overparametrization for smooth sparse
regularization} \label{subsec:hadamard}

Following common practice in sparsifying mixture models,
this paper targets only the mixture weight matrix
$\vectgr{\Xi} \in \Real^{K \times N_a}$ for
sparsification: pushing $\xi_k(a)$ to zero effectively
prunes the $k$-th Gaussian component from the Q-function
corresponding to action $a$, retaining only the most
impactful ones. This choice is not merely conventional
but also structurally necessary: unlike $\vectgr{\Xi}$,
which is unconstrained in Euclidean space, directly
regularizing the means $\vect{M}$ or covariances
$\Set{\vect{C}_k}_{k=1}^{K}$ would distort the geometric
structure of the parameter space $\mathfrak{M}_K$---for
instance, sparsifying $\vect{C}_k$ can violate
positive-definiteness, while sparsifying $\vect{M}$
pushes all Gaussian centers toward the origin, causing
information loss. Formally, the regularized BR
objective~\eqref{eq:BRM.param} takes the form
\begin{equation} \label{eq:reg.obj}
    \min\nolimits_{\vectgr{\Omega} \in \mathfrak{M}_K}
    \Big[ \mathscr{L}_{\mu_n}(\vectgr{\Omega}; Q_n,
    \mathcal{D}_n) \coloneqq
    L_{\mu_n}(\vectgr{\Omega}; Q_n, \mathcal{D}_n)
    + \rho \mathscr{R}(\vectgr{\Xi}) \Big] \,,
\end{equation}
where $L_{\mu_n}(\vectgr{\Omega}; Q_n, \mathcal{D}_n)$ is
defined by~\eqref{eq:BRM.param}, 
$\mathscr{R}(\vectgr{\Xi})$ is a
sparsity-promoting penalty on $\vectgr{\Xi}$ and
$\rho>0$ is a user-defined regularization coefficient.

The ideal sparsity measure is the $\ell_0$-norm
$\norm{\vectgr{\Xi}}_0$, which counts nonzero entries;
however related optimization is NP-hard. The closest
convex relaxation to $\ell_0$-norm, the $\ell_1$-norm,
enables tractable optimization via convexity but also
introduces estimate bias by shrinking large
coefficients and may fail to recover the true support
consistently~\cite{hadamard_sparse}. The quasi-norm
$\norm{\vectgr{\Xi}}_p \coloneqq
\big(\sum_{k=1}^{K}\sum_{a=1}^{N_a} |\xi_k(a)|^p\big)^{1/p}$, with $p \in (0,1)$ offers a
better alternative, tighter approximation to $\ell_0$-norm
and requires weaker conditions for consistent support
recovery, promoting sparser solutions with less
shrinkage bias~\cite{hadamard_sparse}. However, direct
quasi-norm regularization with $p \in (0,1)$ leads to a
non-smooth problem, which is incompatible with
gradient-based optimization, since the sparse solutions are
expect to lie exactly at these non-differentiable points.

To overcome the aforementioned difficulty, this paper
adopts the Hadamard overparametrization
framework~\cite{hadamard_sparse}. In particular, each
scalar weight $\xi_k(a)$ is now replaced by the product
of $J$ auxiliary factors, i.e., $\xi_k(a) \coloneqq
\prod_{j=1}^J \upsilon_{k,j}(a)$, with a user-defined $J
\in \IntegerPP$. Collecting these
into matrices $\vectgr{\Upsilon}_j \in \Real^{K \times
N_a}$ defined entry-wise by $[\Upsilon_j]_{k,a}
\coloneqq \upsilon_{k,j}(a)$, the weight matrix
$\vectgr{\Xi}$ in~\eqref{eq:GMMQF} satisfies the
Hadamard overparametrization
\[
    \vectgr{\Xi} = \odot_{j=1}^{J} \vectgr{\Upsilon}_j
    \,,
\]
where $\odot$ denotes the Hadamard (element-wise)
product. The non-smooth quasi-norm on $\vectgr{\Xi}$ is
then replaced by the smooth surrogate regularizer
\begin{equation} \label{eq:smooth-reg}
    \mathscr{R}(\vectgr{\Xi}) = \mathscr{R}(\vectgr{\Upsilon}_1, \dots,
    \vectgr{\Upsilon}_J) \coloneqq \sum_{j=1}^{J}
    \norm{\vectgr{\Upsilon}_j}^2_{\textnormal{F}} \,,
\end{equation}
where $\norm{}_{\textnormal{F}}$ denotes the Frobenius
norm. It is shown in~\cite[Lemma 10]{hadamard_sparse}
that the optimization via smooth regularizer~\eqref{eq:smooth-reg} in
fact yields similar solutions with the quasi-norm
$\norm{\vectgr{\Xi}}_{2/J}^{2/J}$ if $J > 2$. Notably,
this sparsity of $\vectgr{\Xi}$ emerges implicitly
through gradient-based optimization of the smooth
objective~\eqref{eq:reg.obj}, without any explicit
thresholding.

\subsection{Sparse GMM-QFs and learning objective}
\label{subsec:SGMMQF}

Equipped with the Hadamard
factorization~\eqref{eq:smooth-reg}, the novel class of
sparse (S-)GMM-QFs is now defined as:
\begin{align}
    \mathcal{Q}^{(J)}_K \coloneqq \Bigl \{
    Q \colon \mathfrak{S} \times \mathfrak{A} \to \Real \colon (\vect{s}, a) \mapsto Q(\vect{s}, a) 
    \coloneqq & \sum_{k=1}^{K} \mathop{} \prod_{j=1}^{J}
  \upsilon_{k,j}(a) \mathscr{G}(\vect{s} \given \vect{m}_k, \vect{C}_k) \notag\\
  & \Big \vert\, \upsilon_{k,j}(a) \in \Real, \vect{m}_k \in \Real^{D_s}, \vect{C}_k \in \PD^{D_s}
  \Bigl \} \,. \label{eq:SGMMQF}
\end{align} 
Now, let $\vectgr{\Omega} \coloneqq (\vectgr{\Upsilon}_1,
\dots, \vectgr{\Upsilon}_J, \vect{M}, \vect{C}_1, \dots,
\vect{C}_K)$ parametrize Q-functions in~\eqref{eq:SGMMQF}, 
the corresponding parameter space is
\[
    \mathfrak{M}^{(J)}_K \coloneqq (\Real^{K \times N_a})^J
    \times \Real^{D_s \times K} \times (\PD^{D_s})^K \,.
\]
Altogether, the overall BR objective~\eqref{eq:reg.obj}
is recast to justify with S-GMM-QFs as the following:
\begin{equation}\label{eq:reg.BRM}
    \min\nolimits_{\vectgr{\Omega} \in \mathfrak{M}^{(J)}_K}
    \Big\{ \mathscr{L}_{\mu_n}(\vectgr{\Omega}; Q_n,
    \mathcal{D}_n) \coloneqq
    \frac{1}{T} \sum_{t=1}^{T}
    \Big[
        \sum_{k=1}^{K} \prod_{j=1}^{J} \upsilon_{k,j} (a_t^{(n)})
        \mathscr{G}_k(\vect{s}_t^{(n)})
        - r_t^{(n)} - \alpha
        Q_n (\vect{s}_t'^{(n)}, \mu_n (\vect{s}_t'^{(n)}))
    \Big]^2
    + \rho \sum_{j=1}^{J}
    \norm{\vectgr{\Upsilon}_j}^2_{\text{F}}
    \Big\} \,.
\end{equation}

Instead of learning $\vectgr{\Xi}$ directly as
in~\cite{minh25tsp}, S-GMM-QFs learn the multiplicative
factors $\Set{\vectgr{\Upsilon}_j}_{j=1}^{J}$, enlarging the
parameter space while preserving the functional form
of~\eqref{eq:GMMQF}. Since any scalar $\xi_k(a) \in \Real$ is
trivially attainable as a product of $J$ real components---e.g., by
setting $\upsilon_{k,1}(a) \coloneqq \xi_k(a)$ and
$\upsilon_{k,j}(a) \coloneqq 1$ for $j>1$---the sets of
Q-functions representable by $\mathcal{Q}_K^{(J)}$ and
$\mathcal{Q}_K$ coincide: Hadamard overparametrization changes
only the parametrization of the model, not the underlying
function class, so the universal approximation property of
Proposition~\ref{prop:dense} transfers immediately to
S-GMM-QFs without further proof.

Crucially, this enlarged, redundant parametrization is not
merely a reformulation: it changes the \textit{dynamics} of
gradient-based learning. Under gradient descent on the
regularized objective~\eqref{eq:reg.BRM}, the coupled
multiplicative structure of $\xi_k(a) = \prod_j
\upsilon_{k,j}(a)$ induces an implicit bias toward sparse
solutions, whereby weights corresponding to uninformative
Gaussian components are driven to a magnitude that is
negligible for all practical purposes, effectively pruning the corresponding
component from the Q-function. In the taxonomy
of~\cite{sparseDRL}, S-GMM-QFs training therefore belongs to
the dense-to-sparse kind, akin to
pruning~\cite{dense2sparse}. However, unlike pruning, no
sparsity level or schedule is fixed, the surviving
components emerge naturally from the regularized
objective~\eqref{eq:reg.BRM}. 

Hadamard overparametrization incurs a cost: representing $\xi_k(a)$ via $J$ auxiliary factors inflates
both parameter count and per-step computational complexity. Additionally, while $L_{\mu_n}$ is convex in
$\vectgr{\Xi}$ (with $\vect{M}$ and $\Set{\vect{C}_k}_{k=1}^{K}$ fixed), the multiplicative structure
$\xi_k(a) = \prod_{j=1}^{J} \upsilon_{k,j}(a)$ renders $L_{\mu_n}$ non-convex in
$\Set{\vectgr{\Upsilon}_j}_{j=1}^{J}$, potentially introducing saddle points in the optimization
landscape. However, this trade-off is justified on two grounds: overparametrization provides enhanced
representational capacity, and sufficiently overparametrized systems typically satisfy a
Polyak--{\L}ojasiewicz-type condition in the enlarged parameter space, guaranteeing convergence of
gradient-based methods despite saddle points~\cite{liu22overparams}.

Recall that direct gradient minimization with regularizer
via $\ell_p$ quasi-norm ($p \in (0,1)$) is hindered by its
non-smoothness at the origin, where sparse solutions are
expected to locate.~\cite{hadamard_sparse} shows that this
obstacle can be sidestepped by a \textit{smooth}
objective~\eqref{eq:reg.BRM} to implicitly recover solutions
comparable to those by $\ell_p$-norm itself, with
$p=2/J$. Regularizer via Hadamard
overparametrization~\eqref{eq:smooth-reg} can therefore be
regarded as a smooth surrogate for the usual non-smooth
counterpart, trading an explicit but intractable-to-optimize
penalty for an implicit sparsity-inducing dynamic. 
In practice, the two approaches differ
fundamentally: direct optimization of $\vectgr{\Xi}$ under
$\ell_1$- or $\ell_p$-penalty yields weights that are small
but
remain within a few orders of magnitude of the retained ones,
requiring an explicit thresholding step to separate signal
from noise. From an interpretability standpoint, this
effectively-exact sparsity allows users to transparently
identify, without any auxiliary pruning step, which Gaussian
components are most influential to the Q-function for each
action. The same
implicit-bias mechanism that induces sparsity is also
understood to improve generalization in overparametrized
models more broadly~\cite{hoff2017lasso, li2023tail,
ziyin2023spred, hadamard_sparse}, suggesting that the benefits
of Hadamard overparametrization for S-GMM-QFs extend beyond
interpretability alone.


\section{Riemannian optimization for online
PI}\label{sec:riemannian}
\subsection{Computing Riemannian
gradients}\label{subsec:compute.grad}
This section discusses the basics of Riemannian optimization
used in the proposed framework.
In particular, computations on the manifold
$\mathfrak{M}^{(J)}_K$
requires a Riemannian metric~\cite{RobbinSalamon:22,
  Absil:OptimManifolds:08}. To this end, $\forall \vectgr{\Omega} = ( \vectgr{\Upsilon}_1, \ldots,
\vectgr{\Upsilon}_J, \vect{M}, \vect{C}_1, \ldots,
\vect{C}_K ) \in \mathfrak{M}^{(J)}_K$ and $\forall \vectgr{\Lambda}_i
\coloneqq (\vectgr{\Theta}_{i1}, \dots, \vectgr{\Theta}_{iJ}, \allowbreak \vect{P}_i, \vect{X}_{i1}, \dots,
\vect{X}_{iK}) \in T_{\vectgr{\Omega}} \mathfrak{M}^{(J)}_K$, $i\in \overline{1,2}$, define the Riemannian metric
$\innerp{\vectgr{\Lambda}_1}{\vectgr{\Lambda}_2}_{\vectgr{\Omega}} \coloneqq \sum_{j=1}^{J}
\trace(\vectgr{\Theta}_{1j}^\intercal \vectgr{\Theta}_{2j}) + \trace(\vect{P}_1^\intercal \vect{P}_2) + \sum_{k=1}^{K}
\innerp{\vect{X}_{1k}}{\vect{X}_{2k}}_{\vect{C}_k}$, where $\trace(\cdot)$ stands for the trace of a matrix, and
$\innerp{\cdot}{\cdot}_{\vect{C}_k}$ is any user-defined Riemannian metric of $\PD^{D_s}$. Here, the
affine-invariant~\cite{Pennec:Riemannian:19} metric is chosen: $\forall \vect{C} \in \PD^{D_s}$ and $\forall \vect{X}_i
\in T_{\vect{C}} \PD^{D_s}$, $i\in \overline{1, 2}$, let $\innerp{\vect{X}_1}{\vect{X}_2}_{\vect{C}} \coloneqq
\trace(\vect{C}^{-1} \vect{X}_1 \vect{C}^{-1} \vect{X}_2)$.
Moreover, a \textit{retraction}\/ mapping
$R_{\vectgr{\Omega}} (\cdot) \colon T_{\vectgr{\Omega}}
\mathfrak{M}^{(J)}_K \to
\mathfrak{M}^{(J)}_K$~\cite{Absil:OptimManifolds:08} is needed. For $\vectgr{\Lambda} \coloneqq (\vectgr{\Theta}_1, \ldots
\vectgr{\Theta}_J, \vect{P}, \vect{X}_1, \ldots, \vect{X}_K)
\in T_{\vectgr{\Omega}} \mathfrak{M}^{(J)}_K$, it turns out that
$R_{\vectgr{\Omega}}(\vectgr{\Lambda}) \coloneqq (\ldots, R_{\vectgr{\Upsilon}_j} (\vectgr{\Theta}_j), \ldots,
R_{\vect{M}}(\vect{P}), \dots, R_{\vect{C}_k} (\vect{X}_k), \dots)$, with the respective retractions chosen as follows:
$\forall j\in \overline{1,J}$, $\forall k \in \overline{1,K}$,%
\begin{equation}\label{eq:retractions}%
    R_{\vectgr{\Upsilon}_j}(\vectgr{\Theta}_j)  \coloneqq \vectgr{\Upsilon}_j + \vectgr{\Theta}_j\,,
      \qquad R_{\vect{M}}(\vect{P})  \coloneqq \vect{M} + \vect{P}\,, 
      \qquad R_{\vect{C}_k}(\vect{X}_k)  \coloneqq \exp_{\vect{C}_k} (\vect{X}_k) \,,
\end{equation}%
where, under the affine-invariant metric, 
\begin{equation} \label{eq:expC}
    \exp_{\vect{C}_k} (\vect{X}_k) \coloneqq \vect{C}_k^{1/2} \Exp [\vect{C}_k^{-1/2} \vect{X}_k \vect{C}_k^{-1/2}]
    \vect{C}_k^{1/2} \,,
\end{equation}
with $\Exp(\cdot)$ being the matrix exponential~\cite{RobbinSalamon:22, Hall:Lie:03}.

To run line~\ref{line:compute.gradient} in \cref{algo:gmmq}, the following proposition provides the Riemannian gradient
$\grad \mathscr{L}_{\mu_n} = (\ldots, \grad_{ \vectgr{\Upsilon}_j } \mathscr{L}_{\mu_n}, \ldots, \grad_{ \vect{M} }
\mathscr{L}_{\mu_n}, \ldots, \grad_{ \vect{C}_k } \mathscr{L}_{\mu_n}, \ldots ) = \grad L_{\mu_n} +
\grad \mathscr{R}$~\cite{Absil:OptimManifolds:08}.
\begin{subequations}\label{all.gradients}
    \begin{proposition}\label{prop:gradients} (Computing
        gradients) Given a policy $\mu_n$, a Q-function $Q_n$, and a sampled dataset $\mathcal{D}_n$, let
  $\mathscr{L}_{\mu_n} (\cdot; Q_n, \mathcal{D}_n)$ according to \eqref{eq:reg.BRM}.  Consider
  $\vectgr{\Omega} = (\vectgr{\Upsilon}_1, \ldots, \vectgr{\Upsilon}_J, \vect{M}, \vect{C}_1, \ldots, \vect{C}_K) \in
        \mathfrak{M}^{(J)}_K$ and its associated S-GMM-QF $Q_{ \vectgr{\Omega} }$. For convenience, let also $\delta_t^{(n)} \coloneqq Q_{
    \vectgr{\Omega} } (\vect{s}_t^{(n)}, a_t^{(n)}) - r_t^{(n)} - \alpha Q_n (\vect{s}_{t}'^{(n)},
  \mu_n(\vect{s}_{t}'^{(n)}))$. Then, the following hold true.

  \begin{thmlist}

  \item $\forall k\in \overline{1,K}, \forall j\in
      \overline{1, J}, \forall a_t^{(n)}\in \mathfrak{A}$, $\grad_{ \upsilon_{k,j}
    (a_t^{(n)}) } L_{\mu_n} (\vectgr{\Omega}; Q_n, \mathcal{D}_n) = {\partial L_{\mu_n}}
    (\vectgr{\Omega}; Q_n, \mathcal{D}_n) / {\partial
      \upsilon_{k,j} (a_t^{(n)}) }$, with
    \begin{align}\label{eq:dL.dUpsilon}
      \frac{\partial L_{\mu_n}}{\partial
        \upsilon_{k,j} (a_t^{(n)}) } (\vectgr{\Omega}; Q_n, \mathcal{D}_n) =
      \frac{1}{T} 2\delta_t^{(n)}
        \frac{\prod_{j'=1}^{J} \upsilon_{k, j'}(a_t^{(n)}) }{
            \upsilon_{k,j} (a_t^{(n)}) }
      \mathscr{G}_k (\vect{s}_t^{(n)}) \,.
    \end{align}

  \item Let $\xi_k (a_t^{(n)}) \coloneqq \prod_{j=1}^{J}
      \upsilon_{k,j} (a_t^{(n)})$. Then, $\forall k \in \overline{1,K}$, $\grad_{
    \vect{m}_k } L_{\mu_n} (\vectgr{\Omega}; Q_n, \mathcal{D}_n) = {\partial L_{\mu_n}}
    (\vectgr{\Omega}; Q_n, \mathcal{D}_n) / {\partial \vect{m}_k}$, with
    \begin{equation}\label{eq:dL.dM}
      \frac{\partial L_{\mu_n}}{\partial \vect{m}_k} (\vectgr{\Omega}; Q_n, \mathcal{D}_n) =
      \frac{1}{T} \sum_{t=1}^{T} 4\delta_t^{(n)} \xi_k (a_t^{(n)}) \vect{C}_k^{-1} (\vect{s}_t^{(n)} - \vect{m}_k) \mathscr{G}_k
      (\vect{s}_t^{(n)}) \,.
    \end{equation}

  \item $\forall k \in \overline{1, K}$, with $\vect{B}_{tk} \coloneqq \mathscr{G}_k(\vect{s}_t^{(n)}) (\vect{s}_t^{(n)} -
    \vect{m}_k) (\vect{s}_t^{(n)} - \vect{m}_k)^\intercal$ and under the affine-invariant metric for
    $\PD^{D_s}$~\cite{Pennec:Riemannian:19}, $\grad_{ \vect{C}_k } L_{\mu_n} (\vectgr{\Omega}; Q_n,
    \mathcal{D}_n) = {\partial L_{\mu_n}} (\vectgr{\Omega}; Q_n, \mathcal{D}_n) / {\partial \vect{C}_k}$,
    with
    \begin{equation}\label{eq:dL.dC}
      \frac{\partial L_{\mu_n}}{\partial \vect{C}_k} (\vectgr{\Omega}; Q_n, \mathcal{D}_n) = \frac{1}{T}
      \sum_{t=1}^{T} 2\delta_t^{(n)} \xi_k (a_t^{(n)}) \vect{B}_{tk} \in T_{\vect{C}_k} \PD^{D_s} \,.
    \end{equation}

  \item ${\partial \mathscr{R} (\vectgr{\Omega}) } / {\partial \vectgr{\Upsilon}_j} = 
      2 \vectgr{\Upsilon}_j$
      , $\forall j \in \overline{1, J}$.

  \end{thmlist}
\end{proposition}
 \end{subequations}

\begin{proof}
    See~\ref{prop:gradients.proof}.
\end{proof}

\begin{figure}[t]
    \centering
    \includegraphics[width=.7\textwidth]{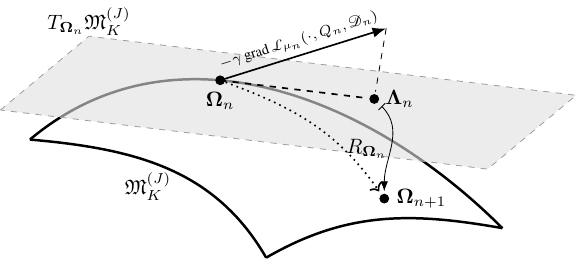}
    \caption[]{ Vanilla gradient-descent update on
    Riemannian manifold $\mathfrak{M}^{(J)}_K$: the Euclidean
    gradient of $\mathscr{L}_{\mu_n} (\vectgr{\Omega}_n, Q_n,
    \mathcal{D}_n)$ is first projected to the tangent space
    $T_{\vectgr{\Omega}_n} \mathfrak{M}^{(J)}_K$ to obtain
    $\vectgr{\Lambda}_n$, then the retraction
    $R_{\Omega_n}(\cdot)$ maps $\vectgr{\Lambda}_n$ back
    onto the manifold to produce $\vectgr{\Omega}_{n+1}$.}
\end{figure}


\subsection{Riemannian gradient-descent}\label{subsec:radam}

Having the Riemannian gradient computed
via~\cref{prop:gradients}, gradient-based
optimization is intrinsically
on $\mathfrak{M}^{(J)}_K$. While vanilla Riemannian gradient
descent~\cite{Absil:OptimManifolds:08} generalizes classical
gradient-descent to curved spaces, it exhibits high variance
and failed to converge reliably in some preliminary
experiments. To address this, 
%
%
a Riemannian extension of the conventional Adam
optimization---coined RAdam---introduced by~\cite{RAdam} is
adopted, incorporating adaptive moment estimation
within the manifold setting. More specially, at iteration
$n$ given $\vectgr{\Pi}_{n-1} \in
T_{\vectgr{\Omega}_{n-1}}\mathfrak{M}^{(J)}_K$ and $\sigma_{n-1}
\in \Real$ available, the refined $\vectgr{\Omega}_{n+1}$ is
obtained as following:
\begin{subequations}\label{eq:update.Omega}
  \begin{alignat}{2}
    \vectgr{\Pi}_n & {} \coloneqq {} && \beta_1 \varphi_{\vectgr{\Omega}_{n-1} \to \vectgr{\Omega}_{n}}
    (\vectgr{\Pi}_{n-1})
    + (1 - \beta_1) \grad \mathscr{L}_{\mu_n} (\vectgr{\Omega}_n; Q_n, \mathcal{D}_n)
    \,, \label{eq:parallel.trans} \\
    \sigma^2_{n} & \coloneqq && \beta_2 \sigma^2_{n-1} 
    + (1 - \beta_2)\norm{\grad \mathscr{L}_{\mu_n}
      (\vectgr{\Omega}_n; Q_n, \mathcal{D}_n)}^2_{\vectgr{\Omega}_n} \,, \\
    \vectgr{\Omega}_{n+1} & \coloneqq &&
      R_{\vectgr{\Omega}_n} [ - \gamma \vectgr{\Pi}_{n}
      \sqrt{1 - \beta_2^n} /
      ( \sigma_{n} (1 - \beta_1^n) ) ] \label{eq:update.5} \,,
  \end{alignat}
\end{subequations}
with user-defined exponential decay rates $\beta_1, \beta_2 \in (0, 1)$,
and a learning rate $\gamma > 0$; while the operator
$\varphi_{\vectgr{\Omega}_{n-1} \to
\vectgr{\Omega}_{n}}(\cdot):
T_{\vectgr{\Omega}_{n-1}}\mathfrak{M}^{(J)}_K \to
T_{\vectgr{\Omega}_{n}}\mathfrak{M}^{(J)}_K$ is the
parallel-transport mapping~\cite{Absil:OptimManifolds:08}. Its inclusion
in~\eqref{eq:parallel.trans} is necessary, because tangent
spaces vary across points on manifold $\mathfrak{M}^{(J)}_K$, and
therefore the momentum $\vectgr{\Pi}_{n-1}$ cannot be
directly combined with the current gradient without
alignment to appropriate tangent space. This construction
ensures that momentum accumulation remains intrinsic to the
geometric structure of $\mathfrak{M}^{(J)}_K$, while the
retraction $R_{\vectgr{\Omega}_n}(\cdot)$ guarantees that
the update iterate remain on the considering manifold.
For the usual Euclidean
components of $\mathfrak{M}^{(J)}_K$, parallel-transport is the
identity mapping, while for $\PD^{D_s}$ and under the
affine-invariant metric, $\forall \vect{X} \in
T_{\vect{C}_{n-1}} \PD^{D_s}$,
\begin{equation*}
    \varphi_{\vect{C}_{n-1} \to \vect{C}_{n}}(\vect{X})
    \coloneqq \vectgr{\Phi}_n \vect{X}
    \vectgr{\Phi}_n^\intercal \in T_{\vect{C}_{n}} \PD^{D_s}\,,
\end{equation*}
where $\vectgr{\Phi}_{n} \coloneqq \vect{C}^{1/2} \Exp \big[
    (1/2) \Log\big( 
    \vect{C}_{n-1}^{-1/2} \vect{C}_n \vect{C}_{n-1}^{1/2}
    \big)
    \big]$, with $\Exp(\cdot)\,,\Log(\cdot)$ being the
matrix exponential and logarithm
mappings~\cite{RobbinSalamon:22, Hall:Lie:03}, while the
retraction mapping $R_{\vectgr{\Omega}}(\cdot)$ is chosen 
following~\eqref{eq:retractions}.

\section{Structures of experience replay
buffers}\label{sec:er}
To enhance data efficiency and stabilize learning, several
strategies for managing and sampling from experience replay
buffer are considered. In~\cite{mnih13dqn}, all stored
transitions $(\vect{s}, a, r, \vect{s}') \in
\mathcal{B}_n$ are treated equally and sampled uniformly from
the buffer, irrespective of their contribution to the
current objective. While this strategy yields unbiased
estimates, it does not differentiate between informative and
less informative experiences. In quest to exploit the
structure of accumulated experience and improve learning
efficiency, several alternative replay strategies that depart from
uniform sampling are considered. These strategies can be
classified into two main categories: distribution-based and
diversity-based.

\subsection{Distribution-based experience replay}
\label{subsubsec:distribution-er}
Distribution-based experience replay samples transitions
from the buffer by imposing a customized PDF onto the
buffer rather than uniformly, with prioritized experience replay (PER)~\cite{PER}
as one of the most celebrated strategies.
PER improves the
efficiency of experience replay by sampling transitions from
the buffer according to their estimated learning importance. In particular, PER considers
transitions $(\vect{s}_b[n], a_b[n], r_b[n], \vect{s}'_b[n])
\in \mathcal{B}_n$, for $b=\overline{1,B}$ with large
TD-error $\delta_b[n] \coloneqq |Q_n(\vect{s}_b[n],a_b[n]) - r_b[n] -
\alpha Q_n(\vect{s}_b'[n], \mu_n(\vect{s}'_b[n]))|$, more
informative, and allocates them higher sampling probability, as these
indicate regions where current estimate of Q-function is
inaccurate. 

Two common variants of PER are proportional PER and
rank-based PER. In proportional PER, priorities are
determined directly as $P_b[n] \coloneqq \delta_b[n] +
\epsilon$, where $\epsilon > 0$; 
while rank-based PER follows a less aggressive approach by
assigning $P_b[n] \coloneqq 1 / \rank_b[n]$, where
$\rank_b[n]$ denotes the position of transition
$(\vect{s}_b[n], a_b[n], r_b[n], \vect{s}'_b[n])$ in the
descended ordering of $\Set{\delta_{b'}[n]}_{b'=1}^{B}$. 
Then each transition is assigned a sampling
probability of 
\begin{equation}
    p_b[n] \coloneqq p(\vect{s}_b[n], a_b[n], r_b[n], \vect{s}'_b[n])
    \coloneqq
    \frac{(P_b[n])^{\alpha_{\text{PER}}}}{\sum_{b'=1}^{B}
    (P_{b'}[n])^{\alpha_{\text{PER}}}} \,,
\end{equation}
where exponent $\alpha_{\text{PER}}$ determines the degrees
of prioritization. Note that, $\alpha_{\text{PER}}=0$
boils to the uniform sampling.

To correct for the bias this non-uniform distribution introduces, conventional PER~\cite{PER} weights
each sampled transition by an importance sampling (IS)
factor $w_b[n] \coloneqq \left[ \, ({1}/{B}) \big/
  {p_b[n]} \, \right]^{\beta_{\textnormal{PER}}} $,
where $\beta \in (0,1]$ controls the degree of
bias correction ($\beta_{\textnormal{PER}}=1$ fully
compensates for the non-uniform sampling). The IS-corrected
BR loss replaces~\eqref{eq:BRM.param} with
\[
    L_{\mu_n}^{\textnormal{IS}}(\vectgr{\Omega}; Q_n,
    \mathcal{D}_n) \coloneqq \frac{1}{T} \sum_{t=1}^{T}
    w_{b_t}[n] \Big[ \sum_{k=1}^{K} \xi_k(a_t^{(n)})
    \mathscr{G}_k(\vect{s}_t^{(n)}) - r_t^{(n)} - \alpha
    Q_n(\vect{s}_t'^{(n)}, \mu_n(\vect{s}_t'^{(n)}))
    \Big]^2 \,,
\]
where $b_t$ denotes the buffer index of the $t$-th
transition sampled into $\mathcal{D}_n$.

While $L_{\mu_n}^{\textnormal{IS}}$ ensures the unbiasedness of the stochastic gradient-descent updates
in standard Euclidean settings, its application within the proposed framework introduces a critical
structural incompatibility.  Specifically, GMM-QFs in~\eqref{eq:GMMQF} are optimized via Riemannian
gradient-descent over their corresponding parameter manifolds, where parameter updates follow geodesic
paths rather than Euclidean translations. Applying the weights $w_b[n]$ in this setting distorts the
manifold geometry, compromising Riemannian retractions on $\mathfrak{M}_K^{(J)}$ and degrading the
validity of manifold-based updates.  This incompatibility is particularly severe for
S-GMM-QFs~\eqref{eq:SGMMQF}, where the manifold structure is intricate and sensitive to perturbations,
resulting in erratic gradient signals and compromised convergence. This sensitivity calls for a remedy
that avoids IS weighting altogether rather than attempting to correct it within the Riemannian setting.


%

To overcome this structural incompatibility, a
frequency-based priority decay mechanism is proposed as an
alternative to IS weight compensation. Rather than modifying
the objective loss $L_{\mu_n}(\cdot; Q_n, \mathcal{D}_n)$ to correct
for sampling bias, the proposed mechanism directly
rebalances the sampling distribution itself, preserving the
original structure of Riemannian optimization problem. In
particular, for each transition $(\vect{s}_b[n], a_b[n],
r_b[n], \vect{s}'_b[n]) \in \mathcal{B}_{n}$, a dedicated
counter $\mathscr{f}_b[n]$ tracks its cumulative selection
frequency. A sampling threshold $\mathscr{F}_{\text{s}} \in
\IntegerPP$ is imposed as a saturation constraint, once a
transition has been sampled more than
$\mathscr{F}_{\text{s}}$ times, a monotonic decay factor is
applied to its priority according to 
\begin{equation} \label{eq:fair.decay}
    P_{b}[n] \coloneqq P_{b}[n-1] \cdot \lambda^{\max(0,
    \mathscr{f}_b[n] - \mathscr{F}_{\text{s}})} \,,
\end{equation}
where $\lambda \in (0,1)$ is a decay rate controlling how
aggressively over-sampled transitions are de-prioritized.
This strategy progressively mitigates the dominance of
frequently revisited transitions, effectively flattening the
overly peaked sampling distributions often induced by
standard PER~\cite{PER}. By re-balancing the utilization of
experiences through priority penalty rather than IS
weights-based loss scaling, this proposed approach maintains
a stable optimization landscape while ensuring that RL
agent remains exposed to a diverse spectrum of environmental
transitions.

%

\subsection{Diversity-based experience replay}
\label{subsubsec:diversity-er}
While distribution-based strategies such as PER prioritize
individual transitions by their estimated learning importance,
they may inadvertently concentrate sampling in narrow regions
of the feature space where TD errors are elevated, neglecting
structurally distinct transitions that are equally relevant
for generalization. This motivates a complementary class of
\textit{diversity-based}\/ replay strategies, which aim to
maintain broad coverage of the experienced feature space
within sampling dataset $\mathcal{D}_n$.

A prominent representative is the efficient diversity-based
experience replay (EDER)~\cite{eder}, which quantifies
diversity of transitions in $\mathcal{B}_n$ via
determinantal point processes~\cite{dpp}. In particular, 
whenever a new transition arrives, it is accepted to the
buffer with a cosine-distance-based probability to reduce
redundancy; the retained transitions are stored sequentially
and partitioned into non-overlapping windows of
consecutive transitions with a user-defined length (set to
$5$ in the experiments of~\cref{sec:tests}). Each
window is scored based on the
determinant of the Gram matrix of pairwise cosine
similarities between its transition. Each transition
inherits its window's score, and is sampled proportionally
to these scores, biasing sampling toward windows whose
transitions collectively cover diverse regions of
state-action space. While effective in several settings,
EDER treats transitions regardless of their TD error,
potentially ignoring the informative ones, especially in
environments with delay-reward effects. A simpler
unsupervised-learning-based alternative via offline K-means
clustering is proposed~\cite{cerLi22}; however operating
K-means periodically over the whole buffer is incompatible
with the single-step streaming setting of~\cref{algo:gmmq}.

This paper follows a different strategy via online
clustering, which promotes diversity by enforcing an equal
sampling quota across clusters, while retaining
informativeness via TD-error-based priority within each
cluster. Given a user-defined $C \in \IntegerPP$, the set of
centroids $\mathfrak{C}[n] \coloneqq \Set{\phi_1[n], \dots,
\phi_C[n]}$ is maintained over a feature representation $\psi_n
\coloneqq \psi(\vect{s}_n, a_n, \vect{s}'_n)$
of each transition. This paper adopts, as one convenient
choice among several possible feature representations, the
concatenation 
\begin{equation*}
    \psi_n \coloneqq [\vect{s}^{\intercal}_n \quad
    \vect{e}^{\intercal}_{a_n} \quad
    {\vect{s}'_n}^{\intercal}]^{\intercal} \in \Real^{2D_s + N_a} \,,
\end{equation*}
where $\vect{e}_{a_n} \in \{0,1\}^{N_a}$ denotes the one-hot
indicator vector of action $a_n$, defined entry-wise by
$[\vect{e}_{a_n}]_{a'} \coloneqq \vect{1}[a' = a_n]$ for
$a' \in \overline{1,N_a}$ ($\vect{1}[\cdot]$ is the
indicator function). 
Every incoming transition $(\vect{s}_n, a_n, r_n,
\vect{s}'_n)$ is
assigned to its nearest centroid, 
\begin{equation}
    c^* \coloneqq \Argmin_{c \in \overline{1,C}}
    d\,(\psi_n, \phi_c[n]) \,,
\end{equation}
where $d(\cdot, \cdot)$ is any user-defined metric; this
paper adopts the simple Euclidean distance.
Only the chosen
centroid is nudged toward the new feature via an exponential
moving average, 
\begin{equation}\label{eq:centroid.update}
    \phi_{c^*}[n+1] \coloneqq (1 - \eta)\, \phi_{c^*}[n] +
    \eta\, \psi_n \,,
\end{equation}
with $\eta \in (0,1)$ to keep the clusters tracking the
evolving distribution of experience without revising the
entire buffer. This online update makes the strategy
compatible with the single-step streaming setting
of~\cref{algo:gmmq}, unlike the periodic K-means
in~\cite{cerLi22}.

At sampling time, an equal quota of $\lfloor {T}/{C} \rfloor$
transitions is drawn from every cluster to guarantee that no explored
region of the state-action space is left out of the sampled
set $\mathcal{D}_n$. Regardless of clusters, transitions
$(\vect{s}_b[n], a_b[n], r_b[n], \vect{s}'_b[n]) \in \mathcal{B}_n$ are
still ranked by their TD-error-based priority $P_b[n]$ as in 
PER~\cite{PER}, but they also carries a counter
$\mathscr{f}_b[n]$ (\cref{subsubsec:distribution-er}),
giving rise to a fairness-adjusted priority
\[
    P_b^{\textnormal{fair}}[n] \coloneqq \frac{P_b[n]}{1 +
    \mathscr{f}_b[n]} \,.
\]
Given the assigned cluster $c$ of transitions
$(\vect{s}_b[n], a_b[n], r_b[n], \vect{s}'_b[n])$, its
within-cluster sampling probability is 
\[
    p_b[n] = p(\vect{s}_b[n], a_b[n], r_b[n], \vect{s}'_b[n])
    \coloneqq
    \frac{(P_b^{\textnormal{fair}}[n])^{\alpha_{\text{PER}}}}{\sum_{b'
    \in \text{cluster
    }c}(P_{b'}^{\textnormal{fair}}[n])^{\alpha_{\textnormal{PER}}}}
    \,,
\]
using the same exponent $\alpha_{\textnormal{PER}}$ as
in~\cref{subsubsec:distribution-er}.

Unlike the threshold-based decay
proposed in~\cref{subsubsec:distribution-er}, this approach applies
continuous priority rescaling from the first re-sample
onward, eliminating the hyperparameters
$\mathscr{F}_{\text{s}}$ and $\lambda$. This simplification
is essential because fairness constraints
must be enforced locally within each cluster pool rather
than globally across the entire buffer, making per-cluster
adaptation more tractable than tuning global decay
parameters.

\section{Numerical tests} \label{sec:tests}
The numerical tests are organized around this paper's
contributions. \Cref{subsec:tests-dense} evaluates
performance of dense GMM-QFs in terms of both data and
computational efficiency. \Cref{subsec:tests-sparse,subsec:tests-interpret} examine the performance and
interpretability gains of sparsification via Hadamard
overparametrization, and \cref{subsec:tests-buffers}
assesses the experience-replay designs for~\cref{algo:gmmq}.

\subsection{Experimental settings}\label{subsec:tests-settings}
Benchmark RL tasks with finite action spaces---the
``Lunar lander'' and
the ``Flappy bird''---are selected to
validate the proposed framework.
\begin{figure}[t]
    \centering
    \subfloat[Lunar lander]{
        \includegraphics[height=.15\paperheight]{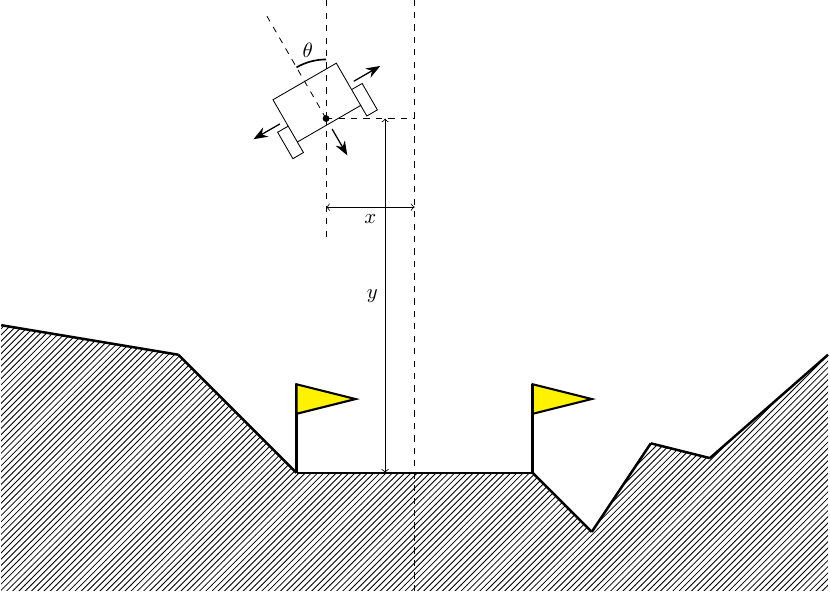}
        \label{fig:lunar}
    }
    \qquad
    \subfloat[Flappy bird]{
        \includegraphics[height=.15\paperheight]{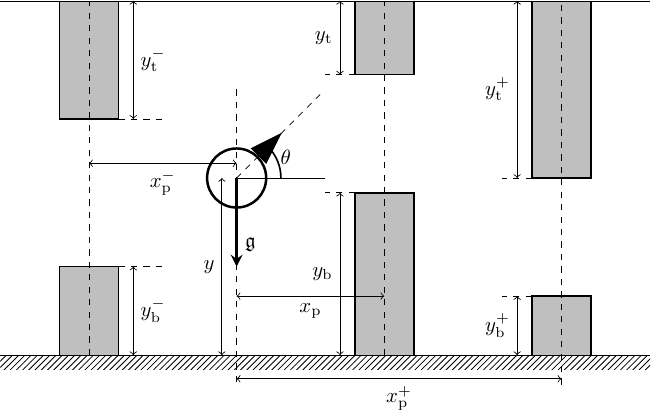}
        \label{fig:flappy}
    }
    \caption[]{RL benchmarks used in~\cref{sec:tests}.}
\end{figure}

The objective of lunar lander task is to control a spacecraft to
land softly on a designated pad using thruster firings
(see~\cref{fig:lunar}). The
state (observation) is an \num{8}-dimensional vector of the
lander's position, linear velocities, angle, angular
velocity and two binary leg-contact indicators. There are
\num{4} available options of action: do nothing, fire the
left engine, fire the main engine (downwards) and fire the
right engine. The reward is dense, increasing as the lander
approaches the desired area at low speed and near-horizontal
orientation. In particular, a reward of $r(\vect{s},a) = 10$
is given per leg in ground-contact; firing engines
incurs a negative reward (penalty) $r(\vect{s},a)=-0.3$ for
main engine, and $r(\vect{s},a)=-0.03$ for side engines. A
terminal bonus of $r(\vect{s},a)=100$ is awarded for a safe
landing, or $r(\vect{s},a)=-100$ for crashing. Episodes
begin with the lander at the top of the viewport subject to
random initial force, and terminate upon landing, crashing,
or leaving the observable screen.

The Flappy Bird task requires controlling a bird to
navigate an endless sequence of pipes under gravity, with
the goal of passing through as many gaps as possible without
collision or leaving the screen boundaries
(see~\cref{fig:flappy}). The state space is 
\num{12}-dimensional, which observes the horizontal
positions and top and bottom vertical positions of the three
nearest upcoming pipes, together with the bird's vertical
position, vertical velocity, and rotation. There are \num{2}
discrete actions: do nothing or flap. The agent receives
a reward of $r(\vect{s},a) = 0.1$ per action for staying
alive and $r(\vect{s},a) = 1$ for each pipe
successfully passed; it is penalized by $r(\vect{s},a) =
-0.5$ for touching the top of the screen and $r(\vect{s},a)=
-1$ upon collision with a pipe or the ground. An episode
terminates when the bird collides with a pipe or the ground,
or leaves the observable screen. 
The
combination of
continuous dynamics, gravity, and strict timing constraints
makes this environment considerably harder than classical
control tasks.

The dense GMM-QFs are compared against
the popular DeepRL methods of DQN~\cite{mnih13dqn} and
PPO~\cite{PPO}; while S-GMM-QFs are compared
against sparsified variants of the DeepRL: \begin{enumerate*}[label=\textbf{(\roman*)}]
\item pruned (dense-to-sparse)~\cite{dense2sparse},
\item SET~\cite{mocanu17scalable}, which updates the sparse network according to a cosine decay criterion~\cite{cosine}, and
\item RigL~\cite{rigl}, which follows the same dynamic sparse-training procedure as SET but regrows connections using the gradient signal rather than at random.
\end{enumerate*}
The DQN~\cite{mnih13dqn} baseline employs double (D)DQN
architecture~\cite{hasselt16ddqn} with prioritized
experience replay~\cite{PER}, and is referred to simply as
DQN hereafter. The DQN learning curves
of~\cref{fig:denses} report, for each benchmark, the
best-performing configuration from the sweep
of~\cref{tab:hyperparams-deeprl}: on lunar lander, a
$2\times 128$ network \emph{without} IS-weight correction
($\beta_{\text{PER}}=0$), which was found to outperform both
the larger networks and the annealed-$\beta_{\text{PER}}$
setting; on flappy bird, a $2\times 512$ network with the
annealed $\beta_{\text{PER}}$. The stronger variant is thus
always reported in favor of the baseline. All remaining DQN
experiments employ the annealed $\beta_{\text{PER}}$
of~\cref{tab:hyperparams-deeprl}.
PPO~\cite{PPO} employs two neural networks---one for the
Q-function (critic) and another for the stochastic policy
(actor)---whereas DQN, similarly to~\cref{algo:gmmq}, models
only the Q-function. In all sparsified PPO variants,
sparsity is applied only to the critic network, while the
actor network remains dense; reported parameter counts
(in~\cref{fig:model-size}) include both networks.
Hyperparameters of~\cref{algo:gmmq} are listed
in~\cref{tab:hyperparams-sgmmqf}, while those of the DeepRL
baselines and their sparsified variants
follow~\cite{sparseDRL} and are collected
in~\cref{tab:hyperparams-deeprl}
of~\ref{subsec:appendix-deeprl-settings}.

The performance metric (vertical axis
in subsequent figures) represents the cumulative reward obtained by the agent until the task is terminated under the
learned policy $\mu_n$, with $n$ denoting the index of the
incoming data (environment transitions) as well as the index of operation for
\cref{algo:gmmq} (horizontal axis). To demonstrate the effect of sparsification, performance is also plotted vs.\ the
number of model's learnable parameters.  Every configuration,
of both \cref{algo:gmmq} and the DeepRL baselines, is
trained over \num{10} independent seeds. Algorithms are
evaluated every \num{5000} incoming data,
with
results averaged over \num{20} independent episodes using a
separate testing emulator for each environment and
aggregated across the \num{10} seeds.
Comparisons against kernel-based and
distributional RL baselines are omitted here: on
flappy bird, \cite{minh25tsp} already reports that GMM-QFs
substantially and consistently outperform these methods,
while \cite{Vu:eusipco:25} documents a similar trend on
simpler benchmark environments. Given the significant
computational cost of these baselines and the consistency of
these prior findings, they are not repeated on lunar lander in
the present study.

\begin{table}[t]
    \centering
    \caption{Hyperparameters of~\cref{algo:gmmq} (S-GMM-QFs).}
    \label{tab:hyperparams-sgmmqf}
    \begin{tabular}{lcc}
        \toprule
        \multirow{2}{*}{Hyperparameter} & \multicolumn{2}{c}{Value} \\
        \cmidrule(l){2-3}
        & Lunar lander & Flappy bird \\
        \midrule
        Number of Gaussian components $K$ & \multicolumn{2}{c}{$\{20, 50, 100, 500\}$} \\
        Number of Hadamard factors $J$ & \multicolumn{2}{c}{$3$} \\
        Buffer capacity $B$ & \multicolumn{2}{c}{$10^5$} \\
        Mini-batch size $T$ & \multicolumn{2}{c}{$64$} \\
        RAdam~\cite{RAdam} learning rate $\gamma$ & \multicolumn{2}{c}{$10^{-3}$} \\
        RAdam decay rates $\beta_1, \beta_2$ & \multicolumn{2}{c}{$0.9, 0.999$} \\
        Regularization coefficient $\rho$ & $\{0.001, 0.005, 0.01, 0.05\}$ &
        $\{10^{-5}, 10^{-4}, 10^{-3}, 10^{-2}\}$ \\
        \midrule
        \multicolumn{3}{l}{\textit{Priority-based sampling (PER, fair, clustering)}} \\
        Priority exponent~\cite{PER} $\alpha_{\text{PER}}$ & \multicolumn{2}{c}{$0.6$} \\
        \midrule
        \multicolumn{3}{l}{\textit{Fair-decay buffer (\cref{subsubsec:distribution-er})}} \\
        Sampling threshold $\mathscr{F}_{\textnormal{s}}$
        in~\eqref{eq:fair.decay} & \multicolumn{2}{c}{$20$} \\
        Decay rate $\lambda$ in~\eqref{eq:fair.decay} & \multicolumn{2}{c}{$0.5$} \\
        \midrule
        \multicolumn{3}{l}{\textit{Clustering buffer (\cref{subsubsec:diversity-er})}} \\
        Number of clusters $C$ & \multicolumn{2}{c}{$5$} \\
        Centroid update rate $\eta$
        in~\eqref{eq:centroid.update} & \multicolumn{2}{c}{$0.05$} \\
        \bottomrule
    \end{tabular}
\end{table}


\subsection{Dense models}\label{subsec:tests-dense}
Performance of dense models is recorded
in~\cref{fig:denses}. Overall,~\cref{algo:gmmq} outperforms
other competitors, with faster improvement over the same
number of observed transitions. 
In~\cref{fig:lunar-tba},
dense GMM-QFs with $K=500$ slightly outperforms their $K=100$
counterpart, consistent with the expectation that a larger
number of Gaussian components offers richer representational
capacity and thus a better approximation of the Q-function.
In~\cref{fig:flappy-tba}, however, this trend does not hold:
$K=100$ outperforms both $K=50$ and $K=500$, with $K=500$
substantially underperforming the two smaller models,
exhibiting higher variance and less stable behavior over the
long horizon.

On the other hand, deep RL
approaches exhibit varying behavior across the two
tasks considered. For lunar lander,
DQN~\cite{mnih13dqn, hasselt16ddqn} and PPO~\cite{PPO} improve steadily
during training.
However, both methods fail to achieve
competitive performance in flappy bird environment:
DQN~\cite{mnih13dqn, hasselt16ddqn} reaches only a suboptimal behavior,
while PPO~\cite{PPO} show little
observable improvement over the course of training.
This is partly attributed to the reliance of on-policy methods
on the quality of data generated by the
current policy: in environments with delayed action effects
such as flappy bird, the policies may not be assessed
efficiently under rollouts of limited length. Moreover,
PPO~\cite{PPO} is designed to learn from multiple rollout
trajectories, making it better suited to simulation-based
training scenarios with a world model available, rather than
a pure online setting where only a single trajectory can be
generated through direct interaction with the environment.

\begin{figure}[t]
    \centering
    \subfloat[Lunar lander]{
        \includegraphics[width=.45\textwidth]{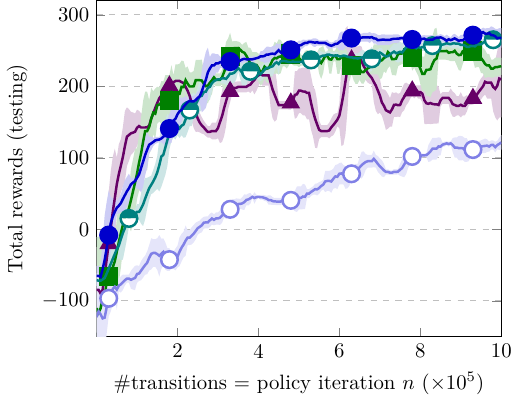}
        \label{fig:lunar-tba}
    }
    \quad
    \subfloat[Flappy bird]{
        \includegraphics[width=.45\textwidth]{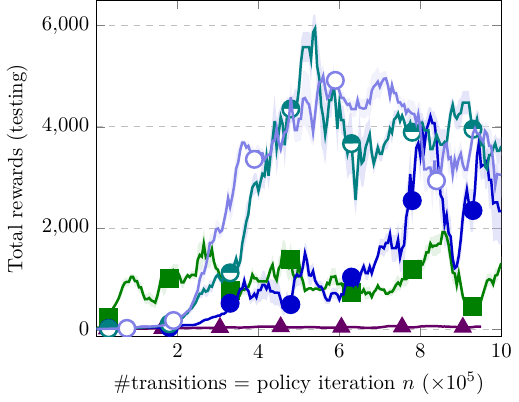}
        \label{fig:flappy-tba}
    }
    \caption[]{Performance of dense models. 
    Results are smoothened using moving interquartile mean,
    window of \num{5e4}. Curve
    markers:~\cref{algo:gmmq}:
    $K=50$:~\quicksymbol{proposed!50}{mark=*, mark
    options={fill=white, scale=1.5, line width=1pt}},
    $K=100$:~\quicksymbol{teal}{mark=halfcircle*, mark
    options={scale=1.5, line width=1pt}},
    $K=500$:~\quicksymbol{proposed}{mark=*, mark
    options={scale=1.5, line width=1pt}};
    DQN~\cite{mnih13dqn, hasselt16ddqn}:~\quicksymbol{dqn}{mark=square*,
    mark options={scale=1.5}}.
    PPO~\cite{PPO}:~\quicksymbol{ppo}{mark=triangle*,
    mark options={scale=2}}; 
    }
    \label{fig:denses}
\end{figure}

\begin{figure}[t]
    \centering
    \subfloat[Lunar lander]{
        \includegraphics[width=.45\textwidth]{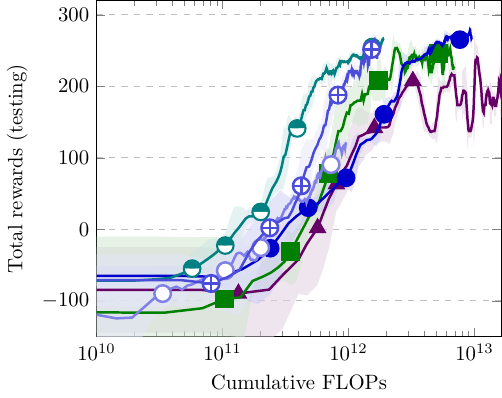}
        \label{fig:lunar-flops}
    }
    \quad
    \subfloat[Flappy bird]{
        \includegraphics[width=.45\textwidth]{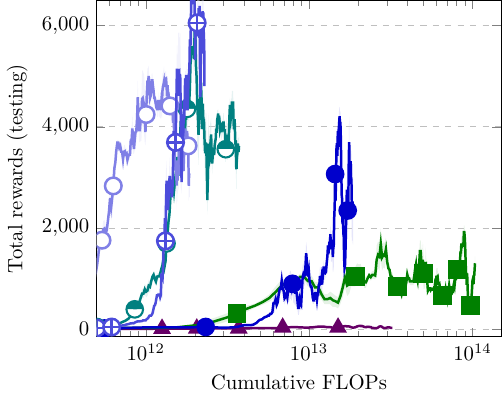}
        \label{fig:flappy-flops}
    }
    \caption[]{Performance of dense models against
    cumulative FLOPs along training. S-GMM-QFs: 
    \tikz[baseline=-0.5ex]{
        \draw[line width=1pt, color=proposed!70, 
        opacity=0.5] (-0.2,0) -- (0.2,0);
        \draw[color=proposed!70, opacity=1]
        plot[mark=quartedcircle, mark options={line width=1pt,
        scale=1.5, fill=white}] coordinates {(0,0)};
    }. 
    Other curve
    markers follow ones of~\cref{fig:denses}.
    }
    \label{fig:denses-flops}
\end{figure}

A FLOP (floating-point operation) is a single elementary
arithmetic operation---such as a multiplication or an
addition---and serves as a hardware-independent proxy for
computational cost. \cref{fig:denses-flops} reports the same
performance curves as~\cref{fig:denses}, but plotted against
cumulative FLOPs along training rather than the number of observed
transitions. For~\cref{algo:gmmq}, FLOPs per policy-evaluation
step are dominated by the Riemannian retraction
\eqref{eq:retractions} on the $K$ covariance matrices, whose
cost scales as $\mathcal{O}(K D_s^3)$, combined with the
$\mathcal{O}(K^2 D_s)$ cost of the Euclidean gradient
computations for the mixture weights and means (a detailed
derivation is given in~\cite[Appendix H]{minh25tsp}); for DQN and PPO,
FLOPs follow the standard convention that a forward and
backward pass together cost a small constant multiple of the
number of network parameters.
It is worth stressing that this computational burden grows
significantly with the state dimensionality $D_s$, since each
Gaussian component maintains a $D_s \times D_s$ covariance
matrix whose inversion---required for evaluating the Gaussian
densities---and Riemannian retraction cost
$\mathcal{O}(D_s^3)$ per component. Consequently,
high-dimensional inputs, such as pixel data or LiDAR
measurements with hundreds of features, are currently beyond
the reach of GMM-QFs, and the benchmarks of moderate $D_s$
adopted in this study reflect this operating regime. Possible
remedies, including state-representation via an online
encoders, are left for future work
(see~\cref{sec:conclusion}).

Under this metric, the ordering observed
in~\cref{fig:denses} is reshaped considerably. On lunar
lander, GMM-QFs with $K=50$ and $K=100$ reach performance
comparable to $K=500$ using one to two orders of magnitude
fewer cumulative FLOPs, since the dominant $K D_s^3$ cost of
the manifold retraction grows directly with the number of
components; DQN and PPO fall in between, requiring FLOPs
comparable to or exceeding the $K=500$ GMM-QF to reach
similar performance. The effect is more pronounced on flappy
bird: $K=50$ and $K=100$ again reach their (respectively
higher) plateaus at comparatively low FLOP counts, whereas
$K=500$ requires close to two orders of magnitude more
compute to reach a considerably lower level of performance,
reflecting the instability already noted
in~\cref{fig:flappy-tba}. DQN reaches a modest reward only
after substantial compute, while PPO fails to improve
appreciably even at the highest FLOP counts considered,
indicating that its poor performance on this benchmark is not
merely a matter of insufficient compute, but a more fundamental
mismatch between the algorithm and the online,
single-trajectory setting of this study. Taken together, these
results suggest that transitions-based sample efficiency
and FLOP-based computational efficiency need not coincide,
and that the leaner S-GMM-QFs are considerably more
compute-efficient than their $K=500$ counterpart, motivating
the sparsification developed in~\cref{sec:sparse-gmmq}.

\subsection{Effect of sparsification}\label{subsec:tests-sparse}
The computational savings anticipated
in~\cref{subsec:tests-dense} are confirmed by the S-GMM-QFs,
whose learning curves are also included
in~\cref{fig:denses-flops}. Although the S-GMM-QFs start from
the same $K=500$ pool of components as its dense
counterpart, the progressive annihilation of mixture weights
along training reduces the associated retraction and
gradient computations, bending its cumulative FLOPs
trajectory well below that of the dense GMM-QFs of
$K=500$ components. On lunar lander, the sparse model
reaches the optimal performance plateau of the dense $K=500$
at a fraction of its cumulative FLOPs, comparable to the
computational budget of the much smaller dense models. The
advantage is more visible on flappy bird, as S-GMM-QFs
attains the highest reward curve earlier in cumulative
FLOPs than the best dense GMM-QFs ($K=100$), and much
earlier than DQN~\cite{mnih13dqn, hasselt16ddqn}, while the $K=500$ never
recovers its cost. Sparsification is thus not merely a
model-compression device applied after training, it acts
during training, converting representational redundancy of a
large initial pool into performance early, then discarding
non-impactful components before their computational cost
accumulates. 

The distinct training regimes behind these sparsification
schemes are visualized in~\cref{fig:nparams-vs-steps}, which
traces the number of learnable parameters along training on
both benchmarks. The dynamic sparse-training methods (SET,
RiGL) maintain the parameter count fixed at initialization,
appearing as flat lines. Pruning and S-GMM-QFs, by contrast,
are both of the dense-to-sparse kind
(cf.~\cref{subsec:hadamard}), yet differ in what drives the
descent: pruning reduces the parameter count according to its
prescribed cubic schedule, terminating exactly at the
user-specified sparsity level, whereas the parameter count of
S-GMM-QFs evolves solely under the implicit bias of the
regularized objective, with the final model size emerging
from the choice of $\rho$ rather than being fixed \emph{a
priori}. This emergent decay is also the mechanism behind the
compute savings observed in~\cref{fig:denses-flops}: as
mixture weights are annihilated, the per-step retraction and
gradient costs shrink accordingly.

\begin{figure}[t]
    \centering
    \subfloat[Lunar lander]{
        \includegraphics[width=.45\textwidth]{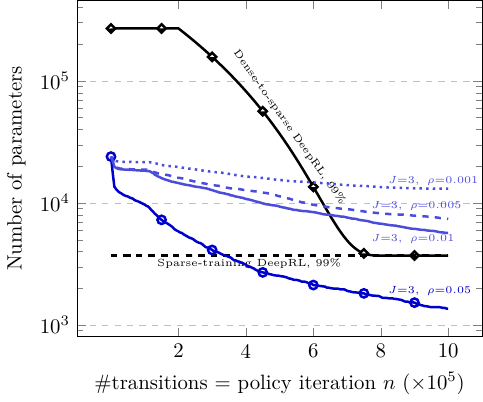}
    }
    \quad
    \subfloat[Flappy bird]{
        \includegraphics[width=.45\textwidth]{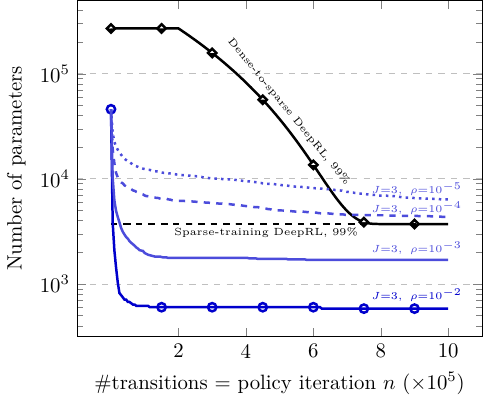}
    }
    \caption[]{Change in number of learnable parameters of
    S-GMM-QFs and sparse DeepRL (DQN~\cite{mnih13dqn,
    hasselt16ddqn},
    PPO~\cite{PPO}) along training. Curve markers:
    \cref{algo:gmmq} with
    Hadamard-overparametrization:\quicksymbol{proposed}{mark=o,
    mark options={line width=1pt}},
    \quickline{proposed!70}{1pt},\quickstyledline{proposed!70}{1pt}{dashed},\quickstyledline{proposed!70}{1pt}{dotted};
    sparse DeepRL:
    pruning~\cite{dense2sparse}:\quicksymbol{black}{mark=halfsquare*, mark
    options={line width=1pt}},
    SET~\cite{mocanu17scalable},
    RiGL~\cite{rigl}:\quickstyledline{black}{1pt}{dashed}.
    }
    \label{fig:nparams-vs-steps}
\end{figure}

Having established the computational benefit, 
S-GMM-QFs are next compared against the pruning, SET, and RigL
sparsification strategies for deep RL networks introduced
in~\cref{sec:intro}, following the taxonomy and
implementations benchmarked by~\cite{sparseDRL}. Unlike the
Hadamard-induced sparsification of S-GMM-QFs, sparsity in all three methods is
imposed directly on the network's connections according to a
user-specified sparsity level or pruning schedule fixed
beforehand, with no explicit link between a surviving
connection and an interpretable region of the state space.
\begin{figure}[t]
    \centering
    \subfloat[Lunar lander]{
        \includegraphics[width=.45\textwidth]{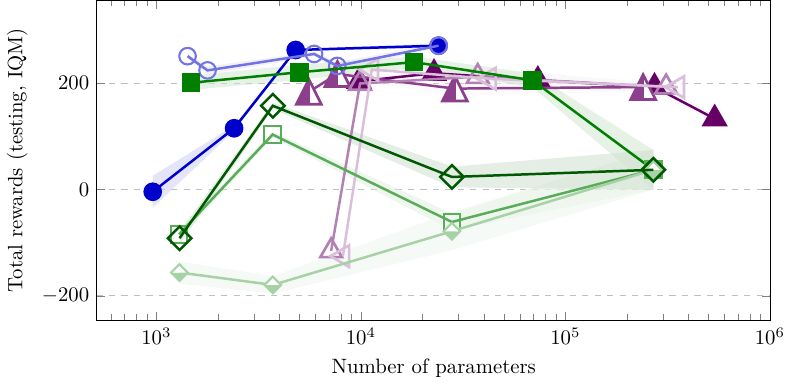}
        \label{fig:lunar-per-vs-params-tba}
    }
    \quad
    \subfloat[Flappy bird]{
        \includegraphics[width=.45\textwidth]{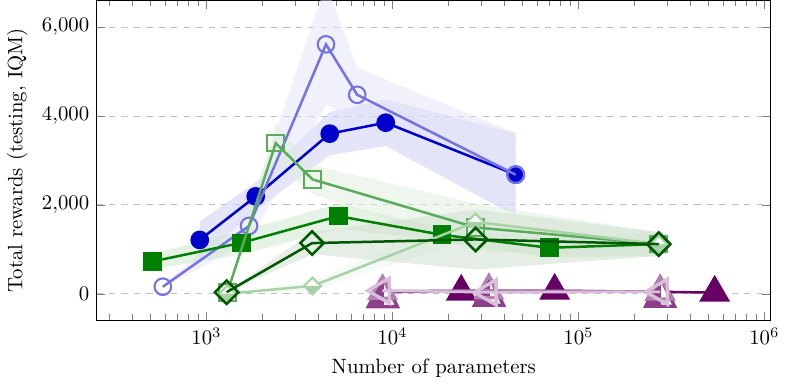}
        \label{fig:flappy-per-vs-params-tba}
    }
    \caption[]{Performance of models against parameter counts.
    Curve markers: \cref{algo:gmmq} (dense, $K$-tuning):
    \proposed, \cref{algo:gmmq} (sparse from $K=500$,
    Hadamard-overparametrization):
    \quicksymbol{proposed!50}{mark=o, mark
    options={scale=1.5, line width=1pt}};
    dense DQN~\cite{mnih13dqn, hasselt16ddqn} (model-tuning): \dqn,
    SET-DQN~\cite{mocanu17scalable}:
    \quicksymbol{dqn!65!white}{mark=square, mark
    options={scale=1.5, line width=1pt}},
    RiGL-DQN~\cite{rigl}:
    \quicksymbol{dqn!70!black}{mark=square, mark
    options={scale=1.5, rotate=45, line width=1pt}},
    Prune-DQN~\cite{dense2sparse}:
    \quicksymbol{dqn!35!white}{mark=halfsquare*, mark
    options={scale=2, line width=1pt}};
    dense PPO~\cite{PPO} (model-tuning): \ppo,
    SET-PPO~\cite{mocanu17scalable}:
    \quicksymbol{ppo!50!white}{mark=triangle, mark
    options={scale=2, line width=1pt}},
    RiGL-PPO~\cite{rigl}:
    \quicksymbol{ppo!25!white}{mark=triangle, mark
    options={scale=2, rotate=90, line width=1pt}},
    Prune-PPO~\cite{dense2sparse}:
    \quicksymbol{ppo!75!white}{mark=halftriangle, mark
    options={scale=2, line width=1pt}};
    }
    \label{fig:model-size}
\end{figure}

\Cref{fig:model-size} reports performance of evaluating
methods against the number of learnable parameters, with
dense models swept over their respective size hyperparameters
($K$ for~\cref{algo:gmmq}, network layer sizes for DeepRL
baselines) and sparse variants swept over sparsity levels.
From the figures, it could be observed that:
\cref{algo:gmmq} is the most parameter-efficient method, as
it attains its peak performance with much fewer parameters
than the amount that DeepRL baselines require to reach their
plateaus. Furthermore, no method improves monotonically with
model size: every curve peaks at an intermediate parameter
count and flattens or degrades beyond this threshold,
indicating that, in the online setting of this study,
additional capacity brings no benefit and may even hinder
the overall performance. 

The relative ordering of the DeepRL baselines, in contrast,
is strongly environment-dependent. In case of dense-reward
lunar lander benchmark, PPO~\cite{PPO} family is competitive
at large model sizes. Prune-PPO~\cite{dense2sparse} remains
functional throughout the low-parameter regime; whereas
SET-PPO~\cite{mocanu17scalable} and RiGL-PPO~\cite{rigl} are
non-functional below the threshold of \num{1e4} parameters
and jump abruptly to the performance level of dense
PPO~\cite{PPO}. A plausible explanation is that
dense-to-sparse training trains the critic densely before
gradually removing weights (un-wiring), so the actor is
always updated with a reliable value estimate, whereas the
dynamic sparse training methods (SET, RiGL) learn a sparse
critic from scratch, and the low-quality Q-function estimates
emerge early. On flappy bird, however, the entire PPO-family
fails at every size considered, even with a dense actor,
corroborating the mismatch between on-policy learning with
delayed-reward settings, which has been discussed
in~\cref{subsec:tests-dense}.

The DQN family, on the other hand, shows a consistent
pattern: dense DQN~\cite{mnih13dqn, hasselt16ddqn} is the most robust
baseline on lunar lander, holding a near-constant plateau
across model-size, yet on flappy bird it is overtaken by its
own sparse-training variants. Similar to PPO case,
Prune-DQN~\cite{dense2sparse} degrades the most gracefully in the
extreme low-parameter regime, where SET-DQN collapses.
Notably, SET-DQN behaves nonmonotonically across sparsity
levels on lunar lander. Inspection of the training curves
reveals that this is not a failure to learn: at sparsity
level $0.9$, SET-DQN quickly obtains optimal performance
before degrading in the late-training phase. Higher
sparsity levels, slower yet stable over the training horizon,
thereby appear superior.

Finally, the proposed S-GMM-QFs (sparsified from $K=500$)
matches its dense counterpart on lunar lander and even
surpasses on flappy bird benchmark, albeit with visible
large variance. This phenomenon where a sparse model
outperforms its dense parent is not unique to S-GMM-QFs, as
SET-DQN also exceeds dense-DQN on flappy bird benchmark.
However, the distinction lies in how the sparsity is
conducted. While SET~\cite{mocanu17scalable} mainly
focuses on the modification within a black-box neural
network, S-GMM-QFs with Hadamard overparametrization carries
a state-relevant interpretability, a mechanism examined
further in~\cref{subsec:tests-interpret}.

\subsection{Interpretability of the learned
representation}\label{subsec:tests-interpret}

Beyond favorable performance-parameter trade-offs,
S-GMM-QFs confer a further, qualitative
benefit: the surviving Gaussian components admit a direct
geometric interpretability in the state space. 

\begin{figure}[t]
    \centering
    \begin{minipage}{\textwidth}
    \centering
    \subfloat[Dense $K=500$]{
        \includegraphics[
        width=.2\textwidth]{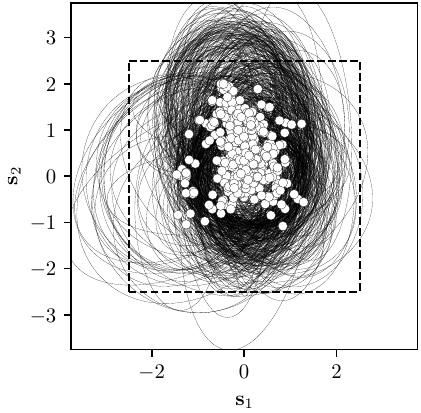}
    }
    \quad
    \subfloat[Dense $K=100$]{
        \includegraphics[
        width=.2\textwidth]{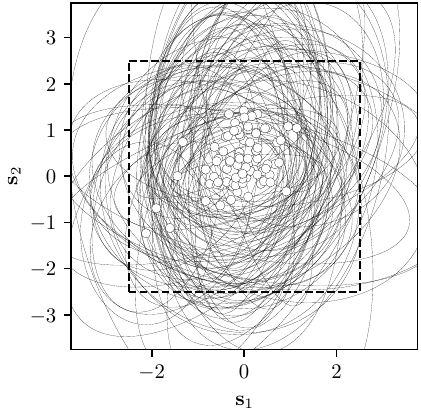}
    }
    \quad
    \subfloat[Dense $K=50$]{
        \includegraphics[
        width=.2\textwidth]{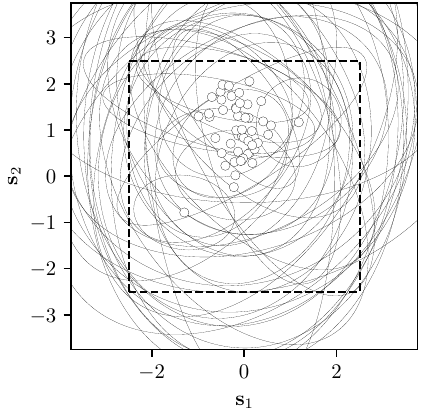}
    }
    \quad
    \subfloat[Dense $K=20$]{
        \includegraphics[
        width=.2\textwidth]{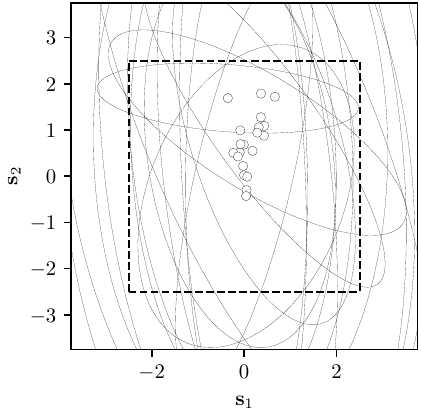}
    }
    \caption[]{Distribution of Gaussian components in dense
        GMM-QFs on lunar lander
        benchmark, shared across all actions. Multivariate GMMs are projected onto
        lander's position coordinates, with ellipses drawn at
        the 63.2\% confidence interval. Dashed box denotes the state domain.
    \dotsymbol{black}{mark=o, mark options={line width=1pt}}
    denotes the Gaussian centers.}
    \label{fig:visual-dense-gmm}
    \end{minipage}

    \vspace{2em}

    \begin{minipage}{\textwidth}
    \centering
    \subfloat[Action 1 (9 Gaussians)]{
        \includegraphics[
        width=.2\textwidth]{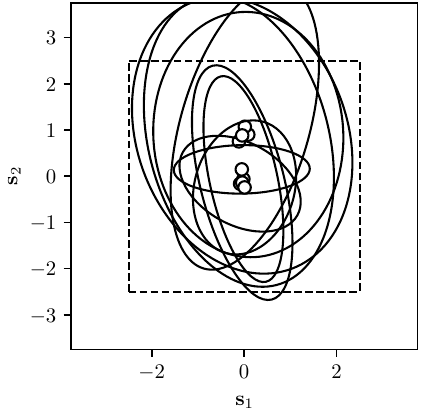}
    }
    \quad
    \subfloat[Action 2 (9 Gaussians)]{
        \includegraphics[
        width=.2\textwidth]{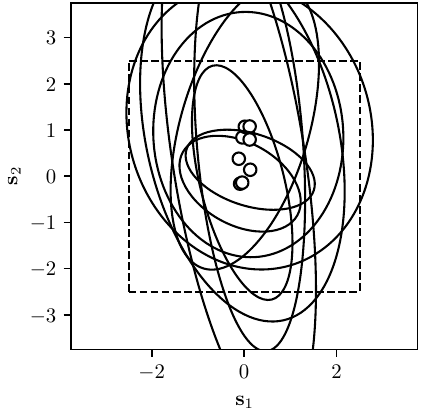}
    }
    \quad
    \subfloat[Action 3 (14 Gaussians)]{
        \includegraphics[
        width=.2\textwidth]{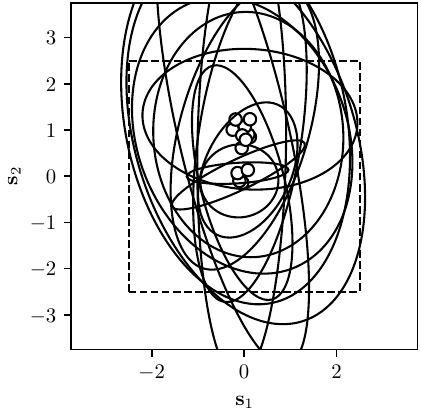}
    }
    \quad
    \subfloat[Action 4 (8 Gaussians)]{
        \includegraphics[
        width=.2\textwidth]{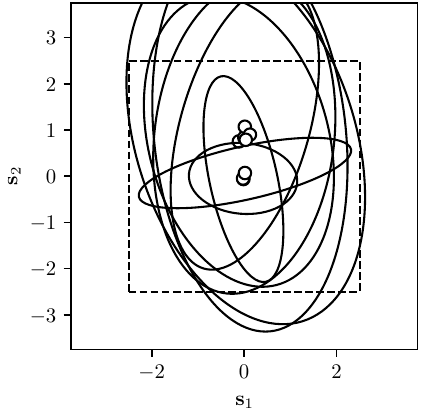}
    }
    \caption[]{Distribution of Gaussian components in
        S-GMM-QFs on lunar lander
        benchmark, sparsified from initial $500$ Gaussian components
        via Hadamard overparametrization ($J=3\,,\rho=0.05$).
        Each action induces a distinct distribution of
        active Gaussian
        components.
        Visualization setting follows ones
        of~\cref{fig:visual-dense-gmm}.
    }
    \label{fig:visual-sparse-gmm}
    \end{minipage}
\end{figure}

\begin{figure}[t]
    \centering
    \centering
    \subfloat[Random policy]{
        \includegraphics[width=.3\textwidth]{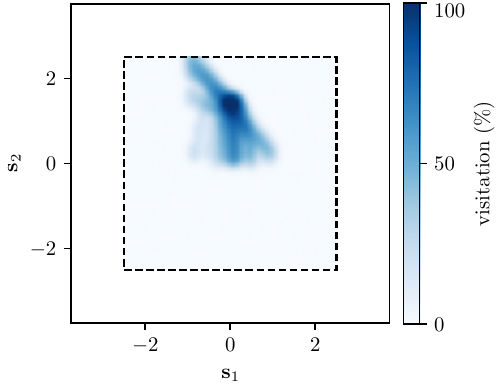}\label{fig:heatmap-random}
    }
    \quad
    \subfloat[Dense GMM-QFs ($K=500$)]{
        \includegraphics[width=.3\textwidth]{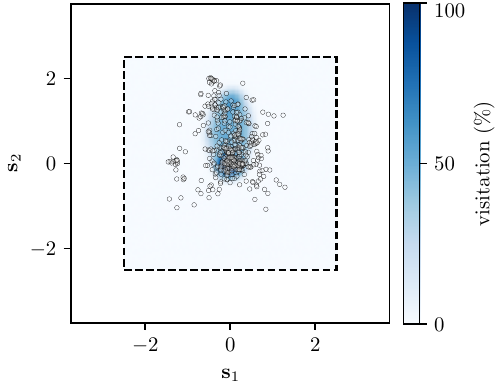}\label{fig:heatmap-dense}
    }
    \quad
    \subfloat[S-GMM-QFs ($K=500\,,J=3\,,\rho=0.05$)]{
        \includegraphics[width=.3\textwidth]{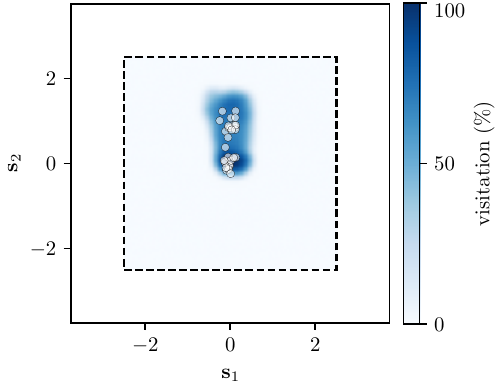}\label{fig:heatmap-sparse}
    }
    \caption[]{State visitation density and distribution of
    Gaussian component centers under three conditions of
    Q-function on
    lunar lander benchmark, projected onto position
    coordinates. Heatmaps are derived from trajectories
    induced from the corresponding policies.
    }
    \label{fig:heatmap-lunar}
\end{figure}


\cref{fig:visual-dense-gmm} shows the distribution of
active Gaussian components in dense GMM-QFs while
\cref{fig:visual-sparse-gmm} shows their
sparsified S-GMM-QF counterparts. In the dense setting,
all $K$ components remain active throughout training, with
weights spread across the mixture regardless of their
individual contribution to the Q-function. By contrast,
S-GMM-QFs trained with~\cref{algo:gmmq} drive a large
fraction of mixture weights $\xi_k(a)$ toward zero,
concentrating the model representation on a small subset
of dominant Gaussian components. The resulting effective
model complexity is comparable to that of dense GMM-QFs
with $K=20$ or $K=50$, despite starting from $K=500$
components.

This comparison reveals an important distinction between
sparsification and simply using a smaller $K$. A dense
GMM-QF with $K=20$ or $K=50$ commits to a fixed number
of components from the outset, limiting its ability to
explore the state space during early training. In contrast,
S-GMM-QFs begin with a richer pool of $K=500$ components
and let the optimization adaptively identify which ones
are most relevant, effectively pruning the rest. The
surviving components therefore tend to be better positioned
and shaped to capture the structure of the Q-function,
which explains why S-GMM-QFs can match or exceed the
performance of dense models with the same or larger $K$,
as observed in~\cref{fig:denses}.

From an interpretability standpoint, the sparsified
representation is significantly more transparent. With
only a handful of active components, each Gaussian can be
associated with a distinct region of the state space,
making its individual contribution to the Q-function
estimates readily identifiable. This is in stark contrast
to the dense $K=500$ model, where the role of any single
component is obscured by the collective contribution of
all others.

\Cref{fig:heatmap-lunar}, on the other hand, constructs additional argument for the interpretive value of
GMM-QFs.  While the random policy establishes a geometric baseline (\cref{fig:heatmap-random}): its state
visitation is unorganized and carries no meaningful information about the task, reflecting a policy with
no learned preference over the state space. The dense GMM-QFs in~\cref{fig:heatmap-dense} stands in clear
contrast. Both the induced policy's visitation and the Gaussian component centers concentrate toward the
task-relevant region, implying that the model's internal geometry directly reflects learning task
structure. Crucially, this is readable from the model itself, without requiring any post-hoc attribution
or external analysis tool.  \Cref{fig:heatmap-sparse} further shows that S-GMM-QFs induces a visitation
heatmap closely resembling that of the dense GMM-QFs, yet with far fewer active components.  More
importantly, the surviving centers remain aligned with S-GMM-QFs' own state visitation, demonstrating
that Gaussian centers faithfully track where the active state region of the induced policy, at any level
of compression.

Such geometric self-documentation is absent in DQNs and their sparse variants. Although these methods
achieve sparsity like S-GMM-QFs, their sparsification targets network connections, which lack inherent
meaning. Consequently, associating surviving connections with meaningful state-space regions requires
external explanation techniques such as saliency maps~\cite{milani24xrl}. By contrast, S-GMM-QFs with
Hadamard overparametrization eliminate this extra step: sparsification operates directly on Gaussian
components defined by explicit geometric parameters (means and covariances), so interpretability emerges
intrinsically from the sparsification process itself, requiring no post-hoc attribution. The same pattern
of component sparsification and visitation-aligned Gaussian centers is observed on the flappy bird
benchmark (\ref{sec:additional-results}), confirming that the interpretability benefit of S-GMM-QFs is
not specific to a single environment.

\subsection{Effect of replay buffers}\label{subsec:tests-buffers}
\begin{figure}[t]
    \centering
    \subfloat[Lunar lander]{
        \includegraphics[width=.45\textwidth]{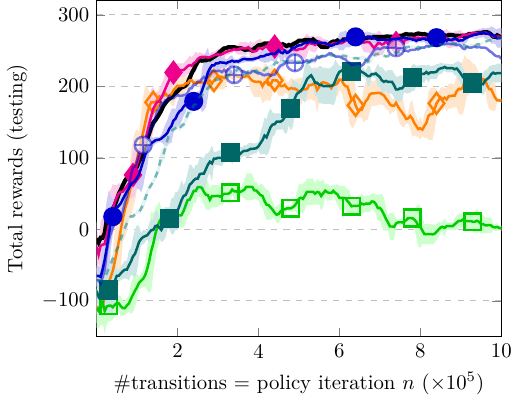}
        \label{subfig:lunarlander-buffers}
    }
    \quad
    \subfloat[Flappy bird]{
        \includegraphics[width=.45\textwidth]{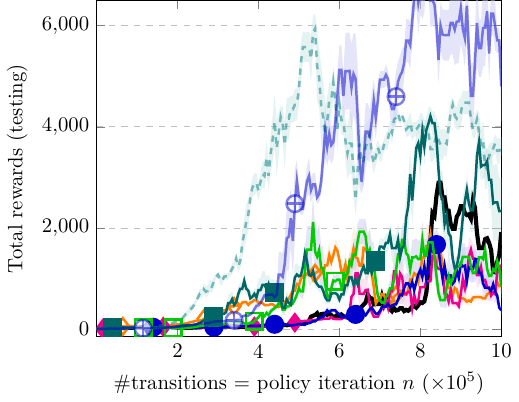}
        \label{subfig:flappy-buffers}
    }
    \caption[]{Performance of dense GMM-QFs ($K=500$) with different
    choices of experience buffer. Curve markers:
    uniform (baseline):\quickline{black}{2pt},
    fair (proposed,~\cref{subsubsec:distribution-er}):\proposed,
    PER~\cite{PER}:~\quicksymbol{orange}{mark=diamond, mark
    options={line width=1pt, scale=2}},
    ranked
    PER~\cite{PER}:\quicksymbol{magenta}{mark=diamond*, mark
    options={line width=1pt, scale=2}},
    EDER~\cite{eder}:\quicksymbol{green!80!black}{mark=square,
    mark options={line width=1pt, scale=1.5}},
    clustering
    (proposed,~\cref{subsubsec:diversity-er}):\quicksymbol{teal!80!black}{mark=square*,
    mark options={scale=1.5}}, 
    dense GMM-QFs of $K=100$ (fair
    in~\cref{subfig:lunarlander-buffers}, clustering
    in~\cref{subfig:flappy-buffers}): \tikz[baseline=-0.5ex]{
         \draw[dashed, line width=2pt, color=teal,
         opacity=0.5] (-0.3,0) -- (0.3,0);},
    sparse S-GMM-QFs of $K=500\,,J=3$ (fair, $\rho=0.05$
    in~\cref{subfig:lunarlander-buffers}; clustering, $\rho=0.0001$
    in~\cref{subfig:flappy-buffers}): \tikz[baseline=-0.5ex]{
         \draw[line width=1pt, color=proposed, 
         opacity=0.5] (-0.2,0) -- (0.2,0);
        \draw[color=proposed, opacity=0.5]
        plot[mark=quartedcircle, mark options={line width=1pt,
        scale=1.5, fill=white}] coordinates {(0,0)};
    }.
    }
    \label{fig:buffers-compare}
\end{figure}

\Cref{fig:buffers-compare} compares the six buffer/sampling
configurations on lunar lander and flappy bird benchmark. For
consistency, importance-sampling (IS) weighting is disabled
across all buffer variants compared here, including standard
and rank-based PER, so that the reported differences in
performance for dense GMM-QFs reflect the sampling
distribution alone rather than the presence or absence of
IS-weight correction. On
lunar lander~\cref{subfig:lunarlander-buffers}, the uniform
baseline and the proposed fair decay scheme track each other
almost exactly, rising fastest among all curves and
converging to the high plateau; while rank PER~\cite{PER}
follows closely behind. Proportional PER~\cite{PER} and
clustering (proposed) converge to a visibly sub-optimal
behavior, with clustering buffer approaching this level only
in later stages of training, while EDER~\cite{eder} fails to
score desired performance. Notably, the fair scheme paired
with a sparse S-GMM-QF ($K=500\,,J=3\,,\rho=0.05$) tracks the
uniform baseline and its dense counterpart in the later
stages of training, despite relying on far fewer effective
(nonzero-weight) components. This suggests that, in a task
with dense, well-shaped rewards, distribution-based
prioritization already exploits the buffer efficiently,
leaving comparatively little room for diversity-based
sampling to improve upon, and Hadamard sparsification
preserves this efficiency at a fraction of the model size.

On the other hand, for the flappy bird benchmark
(\cref{subfig:flappy-buffers}), the ordering changes
substantially: with $K=500$, fair (proposed), PER
(proportional, rank)~\cite{PER}, and EDER~\cite{eder} all
remain flat and mutually comparable, plateauing well below a
desired cumulative reward (\num{2000}) with significant
variance; while the uniform baseline and dense clustering
break away in the later stage of training. The proposed
clustering buffer paired with a leaner dense GMM-QF
($K=100$) improves further still, outperforming all these
configurations for most of the training horizon. Strikingly,
pairing clustering with the sparse S-GMM-QF
($K=500\,,J=3\,,\rho=0.0001$) yields the best performance of
all, rising earliest and ultimately surpassing even the
$K=100$ dense clustering variant. This indicates that
sparsification is not merely matching what a smaller
fixed-$K$ model achieves, but exceeding it: by starting from a
large pool of $K=500$ candidate Gaussians and adaptively
pruning them alongside diversity-driven sampling, S-GMM-QFs
retain the flexibility to explore a delayed, highly redundant
reward landscape early in training while still converging to
a compact, effective representation---a combination
unavailable to a model whose component count is fixed a
priori. This observation is consistent with the earlier
finding that a dense model of $K=500$ Gaussian components is
substantially less stable than the smaller $K=50$ and $K=100$
models in the flappy bird task, and further
suggests that adaptive sparsification, rather than simply
shrinking $K$, is the appropriate remedy. We also note that
the regularization strength required to realize this benefit
differs sharply between tasks ($\rho=0.05$ on lunar lander
versus $\rho=0.0001$ on flappy bird), reflecting flappy bird's
greater reliance on retaining representational capacity even
as redundant components are pruned. Notably, the
distribution-based strategies, including the proposed fair
buffer scheme, fail to match even the uniform baseline,
suggesting that TD-driven prioritization offers little
benefit, and possibly even a mild disadvantage, when the
transitions are redundant with delayed reward effect.

Taken together, the results indicate that no single buffer
strategy dominates universally; distribution-based replay
buffer is effective on dense-reward setting, while
diversity-based clustering is essential on flappy bird's
delayed-reward settings. Moreover, this benefit compounds
with Hadamard sparsification: rather than a mere
model-compression device, sparsification via S-GMM-QFs acts
as a complementary mechanism to diversity-based experience
replay, jointly delivering the best performance observed on
the more challenging flappy bird benchmark while remaining competitive,
at a fraction of the effective parameters, on the easier
lunar lander.

\section{Conclusions}\label{sec:conclusion}


This paper introduced sparse Gaussian-mixture-model Q-functions (S-GMM-QFs), extending prior offline work
into an online, off-policy policy-iteration framework. Streaming transitions were accumulated into an
experience buffer, with several buffer structures studied to balance exploration and sample
efficiency. Parameters were updated via Riemannian optimization on a smooth manifold structure,
respecting the geometric constraints of the parameter space.

Model complexity was controlled through sparsification via Hadamard overparametrization, enabling smooth
regularization compatible with Riemannian-based optimization rather than non-smooth $\ell_p$-norm
penalties. The proposed approach differs fundamentally from simply choosing a smaller number of
components: a model initialized with a large pool retained flexibility to explore the state space early
in training, then progressively concentrated on meaningful components. This adaptive selection yielded
interpretability naturally: each surviving component's parameters (means and covariances) explicitly
encoded its geometric role in the ambient state-action space, providing transparency without the post-hoc
explanation tools typically required by sparse deep RL.

Numerical tests on standard RL benchmarks demonstrated that S-GMM-QFs matched or exceeded competing
methods with faster improvement per observed transition. This advantage persisted and often widened in
low-parameter regimes where sparse deep RL approaches degraded substantially.

The current framework is suited to moderate state dimensions but becomes problematic for raw sensory
observations such as images or LiDAR scans: per-component matrix inversion and Riemannian retraction
scale with cubic-order computational complexity ($\mathcal{O}(D_s^3)$), becoming prohibitive as state
dimension grows. These limitations, together with the restriction to discrete-action problems, suggest
clear directions for future work. The Q-function's role as a critic extends naturally to
continuous-action settings via actor-critic frameworks, where preliminary experiments with GMM value
functions, trained under the same Riemannian machinery, already exhibit stable
learning. State-representation learning offers a path to addressing the dimensionality constraint. Both
extensions are under active development.

\appendix
\crefalias{section}{appendix}

\section{Proof
of~\cref{prop:gradients}}\label{prop:gradients.proof}

To simplify notation in the following proofs,
$\mathscr{G}(\cdot \given \vect{m}_k, \vect{C}_k)$ will be
written as $\mathscr{G}_k(\cdot)$ or
$\mathscr{G}_{\vect{C}_k}$. The dataset $\mathcal{D}_n$, as
well as the current estimate $Q_n$, will
be suppressed in the loss notation. The derivation
of $\partial \mathscr{R}(\vectgr{\Omega}) / \partial
\vectgr{\Upsilon}_j$ is straightforward and therefore skipped.

For convenience, let $\xi_k(a) \coloneqq \prod_{j=1}^{J}
\upsilon_{k, j}(a)\,,\forall a \in \mathfrak{A}$. Thus, we
have the following:
\begin{equation*}
    L_{\mu_n}(\vectgr{\Omega}) \coloneqq \frac{1}{T}
    \sum_{t=1}^{T} \bigg[
    \sum_{k=1}^{K} \xi_k(a_t) \mathscr{G}_k(\vect{s}_t) -
    r_t - \alpha Q_n (\vect{s}_t', \mu_n(\vect{s}_t))
    \bigg]^2
    \\
    = \frac{1}{T} \sum_{t=1}^{T} \delta_t^2
\end{equation*}

\subsection{Derivation of~\eqref{eq:dL.dUpsilon}}
First, let us consider $\delta_t(\vectgr{\Omega}) \coloneqq
\sum_{k=1}^{K} \xi_k(a_t) \mathscr{G}_k(\vect{s}_t) -
r_t - \alpha Q_n (\vect{s}_t', \mu_n(\vect{s}_t))$.
It is trivial that, $\partial \delta_t (\vectgr{\Omega})/ \partial \xi_k(a_t)
= \mathscr{G}_k(\vect{s}_t)$, and $\partial \xi_k(a_t) /
\partial \upsilon_{k, j}(a_t) = \prod_{j'=1}^{J}
\upsilon_{k, j'}(a_t) / \upsilon_{k, j}(a_t)$. Applying the
standard chain rule,
\begin{align*}
    \frac{\partial L_{\mu_n}}{\partial \upsilon_{k,j}(a_t)}
    (\vectgr{\Omega}) & = \frac{1}{T} \sum_{t=1}^{T} \frac{\partial}{\partial
    \upsilon_{k,j}(a_t)} \delta_t^2(\vectgr{\Omega})
    = \frac{1}{T} \sum_{t=1}^{T} 2\delta_t \frac{\partial
    \delta_t}{\partial
    \upsilon_{k,j}(a_t)} (\vectgr{\Omega}) \\
    & = \frac{1}{T} \sum_{t=1}^{T} 2\delta_t \frac{\partial
    \xi_k(a_t)}{ \partial \upsilon_{k,j}(a_t)} \frac{\partial
    \delta_t}{\partial
    \xi_k(a_t) } (\vectgr{\Omega})
    = \frac{1}{T} \sum_{t=1}^{T} 2\delta_t \frac{\prod_{j'=1}^{J}
    \upsilon_{k, j'}(a_t)}{\upsilon_{k, j}(a_t)}
    \mathscr{G}_k(\vect{s}_t) \,,
\end{align*}
establishing~\eqref{eq:dL.dUpsilon}.

\subsection{Derivation of~\eqref{eq:dL.dM}}
We now consider $\partial \delta_t (\vectgr{\Omega})/ \partial \vect{m}_k$:
\begin{align*}
    \frac{\partial \delta_t}{\partial \vect{m}_k}
    (\vectgr{\Omega}) =
    \xi_k(a_t) \frac{\partial \mathscr{G}_k(\vect{s}_t)}{\partial
    \vect{m}_k} = \xi_k(a_t) 2 \mathscr{G}_k(\vect{s}_t)
    \vect{C}_k^{-1} (\vect{s}_t - \vect{m}_k) \,.
\end{align*}
Applying again the chain rule,
\begin{align*}
    \frac{\partial L_{\mu_n}}{\partial \vect{m_k}}
    (\vectgr{\Omega}) = \sum_{t=1}^{T} 2\delta_t
    \frac{\partial \delta_t}{\partial \vect{m}_k}
    (\vectgr{\Omega}) =
    \sum_{t=1}^{T} 4 \delta_t \xi_k(a_t) \mathscr{G}_k(\vect{s}_t)
    \vect{C}_k^{-1} (\vect{s}_t - \vect{m}_k) \,,
\end{align*}
which establishes~\eqref{eq:dL.dM}.
\subsection{Derivation of~\eqref{eq:dL.dC}}
The (partial) derivative of $\delta_t(\vectgr{\Omega})$ with
respect to $\vect{C}_k$, at the point $\vectgr{\Omega}$ and
along an arbitrarily fixed direction $\vect{X} \in
T_{\vect{C}_k} \PD^{D_s} = \mathbb{S}^{D_S}$, together with
the chain rule of differentiation and the fact that the
derivative of the inverse matrix function $\inv: \PD^{D_s}
\to \PD^{D_s}: \vect{C} \mapsto \inv(\vect{C}) \coloneqq
\vect{C}^{-1}$ along $\vect{X}$ is $\Df
\inv(\vect{C})[\vect{X}] = -\vect{C}^{-1} \vect{X}
\vect{C}^{-1}$, yields
\begin{align*}
    \Df_{\vect{C}_k} \delta_t(\vectgr{\Omega})[\vect{X}] 
    & =
    \xi_k(a_t) \mathscr{G}_k(\vect{s}_t) (\vect{s}_t -
    \vect{m}_k)^{\intercal} \Df \inv(\vect{C}_k)[\vect{X}]
    (\vect{s}_t - \vect{m}_k) \\
    & = 
    \xi_k(a_t) \mathscr{G}_k(\vect{s}_t) (\vect{s}_t -
    \vect{m}_k)^{\intercal} \vect{C}_k^{-1} \vect{X}
    \vect{C}_k^{-1}
    (\vect{s}_t - \vect{m}_k) \\
    & = 
    \xi_k(a_t) \trace \bigg( 
    \vect{C}_k^{-1} \mathscr{G}_k(\vect{s}_t) (\vect{s}_t -
    \vect{m}_k)^{\intercal} (\vect{s}_t - \vect{m}_k)
    \vect{C}_k^{-1} \vect{X} 
    \bigg) \\
    & = \trace \bigg( 
    \vect{C}_k^{-1} \xi_k(a_t) \mathscr{G}_k(\vect{s}_t) (\vect{s}_t -
    \vect{m}_k)^{\intercal} (\vect{s}_t - \vect{m}_k)
    \vect{C}_k^{-1} \vect{X}
    \bigg) \,,
\end{align*}
then
\begin{align*}
    \Df_{\vect{C}_k} L_{\mu_n}(\vectgr{\Omega})[\vect{X}] =
    \frac{1}{T}\sum_{t=1}^{T} 2\delta_t(\vectgr{\Omega}) \xi_k(a_t)
    \trace \bigg( \vect{C}^{-1}_k
    \mathscr{G}_k(\vect{s}_t)
    (\vect{s}_t - \vect{m}_k)^\intercal (\vect{s}_t -
    \vect{m}_k) \vect{C}_k^{-1} \vect{X}
    \bigg) \,.
\end{align*}
Recall now the following identity connecting the gradient
with the derivative~\cite[Appendix
A.5]{Absil:OptimManifolds:08}:
\begin{align}
    \innerp{\frac{\partial L_{\mu_n}}{\partial
    \vect{C}_k}}{\vect{X}}_{\vect{C}_k} &=
    \Df_{\vect{C}_k} L_{\mu_n}(\vectgr{\Omega})[\vect{X}]
    \notag \\
    \Rightarrow \innerp{\frac{\partial L_{\mu_n}}{\partial
    \vect{C}_k}}{\vect{X}}_{\vect{C}_k} &=
    \frac{1}{T}\sum_{t=1}^{T} 2\delta_t(\vectgr{\Omega}) \xi_k(a_t)
    \trace \bigg( \vect{C}^{-1}_k
    \mathscr{G}_k(\vect{s}_t)
    (\vect{s}_t - \vect{m}_k)^\intercal (\vect{s}_t -
    \vect{m}_k) \vect{C}_k^{-1} \vect{X}
    \bigg) \,, \label{eq:deriv.innerp}
\end{align}
where $\innerp{\cdot}{\cdot}_{\vect{C}_k}$ stands for the
adopted Riemannian metric. This paper employs the
affine-invariant~\cite{Pennec:Riemannian:19},~\eqref{eq:deriv.innerp}
yields
\begin{equation*}
    \trace\bigg( \vect{C}^{-1}_k \frac{\partial
    L_{\mu_n}}{\partial \vect{C}_k} \vect{C}^{-1}_k
    \vect{X} \bigg) = \frac{1}{T}\sum_{t=1}^{T}
    2\delta_t(\vectgr{\Omega}) \xi_k(a_t)
    \trace \bigg( \vect{C}^{-1}_k
    \mathscr{G}_k(\vect{s}_t)
    (\vect{s}_t - \vect{m}_k)^\intercal (\vect{s}_t -
    \vect{m}_k) \vect{C}_k^{-1} \vect{X}
    \bigg) \,,
\end{equation*}
and because $\vect{X} \in T_{\vect{C}_k} \PD^{D_s}$ is
chosen arbitrarily,
\[
    \frac{\partial
    L_{\mu_n}}{\partial \vect{C}_k} (\vectgr{\Omega})= \frac{1}{T}
    \sum_{t=1}^{T} 2\delta_t(\vectgr{\Omega}) \xi_k(a_t)
    \vect{B}_{tk} \,,
\]
with $\vect{B}_{tk}$ defined in~\cref{prop:gradients},
establishing~\eqref{eq:dL.dC}.

\section{Additional Results}\label{sec:additional-results}
This appendix collects supplementary results and
reproducibility details supporting~\cref{sec:tests}.

\subsection{Interpretability on the flappy bird
benchmark}\label{subsec:appendix-interpret-flappy}
This subsection reports, on the flappy
bird benchmark, the counterparts of the Gaussian-component
distributions and state-visitation heatmaps discussed for
lunar lander
(\cref{fig:visual-dense-gmm-flappy,fig:visual-sparse-gmm-flappy,fig:heatmap-flappy}):
these corroborate the same interpretability claims made
via~\cref{fig:visual-dense-gmm,fig:visual-sparse-gmm,fig:heatmap-lunar}
without introducing new observations, and the corresponding
discussion in~\cref{subsec:tests-interpret} applies
throughout.

\begin{figure}[t]
    \centering
    \begin{minipage}{\textwidth}
    \centering
    \subfloat[Dense $K=500$]{
        \includegraphics[
        width=.2\textwidth]{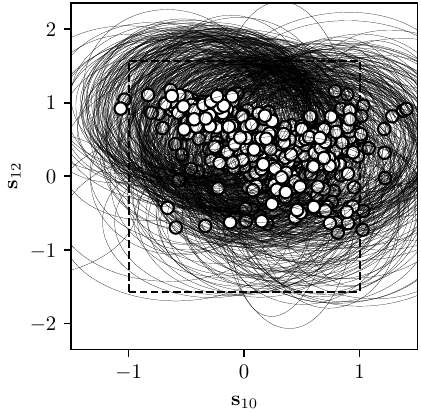}
    }
    \quad
    \subfloat[Dense $K=100$]{
        \includegraphics[
        width=.2\textwidth]{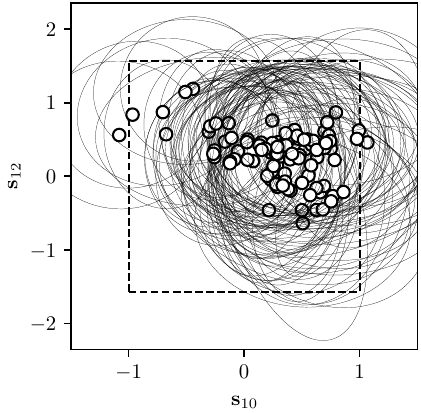}
    }
    \quad
    \subfloat[Dense $K=50$]{
        \includegraphics[
        width=.2\textwidth]{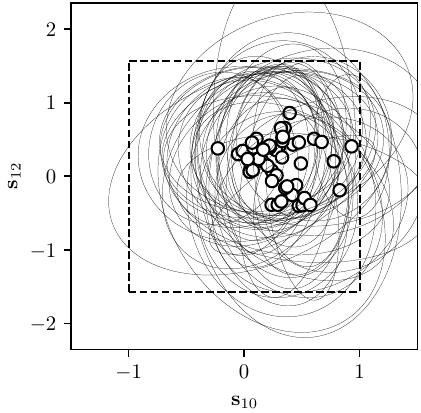}
    }
    \quad
    \subfloat[Dense $K=20$]{
        \includegraphics[
        width=.2\textwidth]{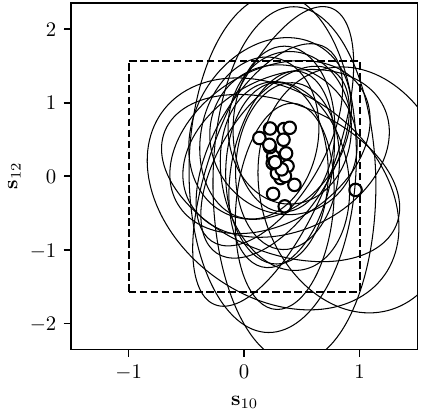}
    }
    \caption[]{Distribution of Gaussian components in dense
        GMM-QFs on flappy bird
        benchmark, shared across all actions. Visualization setting follows ones
        of~\cref{fig:visual-dense-gmm}.
        }
    \label{fig:visual-dense-gmm-flappy}
    \end{minipage}

    \vspace{2em}

    \begin{minipage}{\textwidth}
    \centering
    \subfloat[Action 1 (8 Gaussians)]{
        \includegraphics[
        width=.2\textwidth]{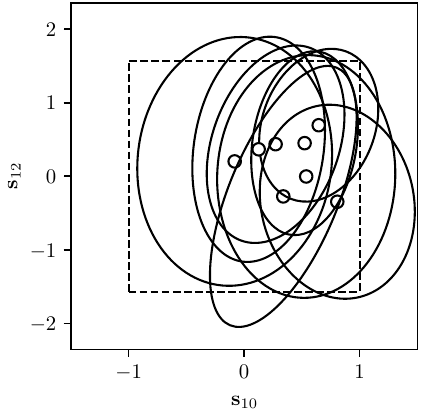}
    }
    \quad
    \subfloat[Action 2 (6 Gaussians)]{
        \includegraphics[
        width=.2\textwidth]{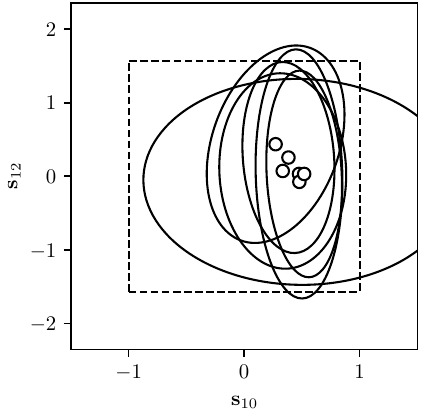}
    }
    \caption[]{Distribution of Gaussian components in
        S-GMM-QFs on flappy bird
        benchmark, sparsified from initial $500$ Gaussian components
        via Hadamard overparametrization ($J=3\,,\rho=10^{-4}$).
        Each action induces a distinct distribution of
        active Gaussian
        components.
        Visualization setting follows ones
        of~\cref{fig:visual-dense-gmm}.
    }
    \label{fig:visual-sparse-gmm-flappy}
    \end{minipage}
\end{figure}

\begin{figure}[t]
    \centering
    \centering
    \subfloat[Random policy]{
        \includegraphics[width=.3\textwidth]{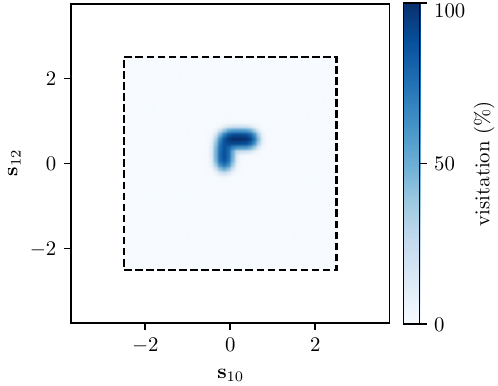}\label{fig:heatmap-random-flappy}
    }
    \quad
    \subfloat[Dense GMM-QFs ($K=100$)]{
        \includegraphics[width=.3\textwidth]{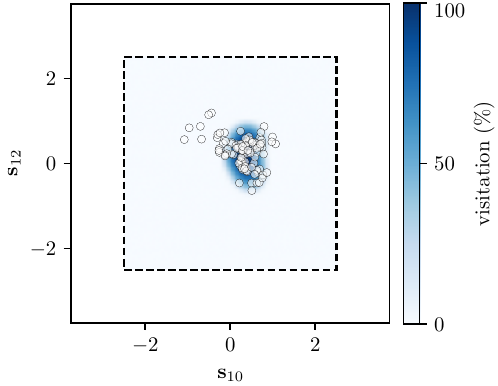}\label{fig:heatmap-dense-flappy}
    }
    \quad
    \subfloat[S-GMM-QFs ($K=500\,,J=3\,,\rho=10^{-4}$)]{
        \includegraphics[width=.3\textwidth]{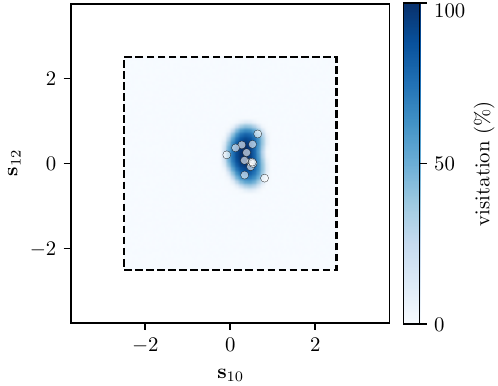}\label{fig:heatmap-sparse-flappy}
    }
    \caption[]{State visitation density and distribution of
    Gaussian component centers under three conditions of
    Q-function on
    flappy bird benchmark, projected onto vertical
    position and velocity. Heatmaps are derived from trajectories
    induced from the corresponding policies.
    }
    \label{fig:heatmap-flappy}
\end{figure}

\section{Hyperparameters of the DeepRL
baselines}\label{subsec:appendix-deeprl-settings}
\Cref{tab:hyperparams-deeprl} lists the hyperparameters of
the DeepRL baselines and their sparsified variants used
in~\cref{sec:tests}, following the settings benchmarked
by~\cite{sparseDRL}.

\begin{table}[!tp]
    \centering
    \caption{Hyperparameters of the DeepRL baselines.}
    \label{tab:hyperparams-deeprl}
    \begin{tabular}{l>{\raggedright\arraybackslash}p{.38\textwidth}}
        \toprule
        Hyperparameter & Value \\
        \midrule
        \multicolumn{2}{l}{\textit{Network architecture}} \\
        DQN Q-network (hidden layers $\times$ width)
        & $2\times \{512, 256, 128, 64\}$ \\
        PPO critic network (hidden layers $\times$ width) &
        $2\times\{512, 256, 128, 64\}$\\
        PPO actor network (hidden layers $\times$ width) &
        $2\times 64$ \\
        \midrule
        \multicolumn{2}{l}{\textit{Optimization}} \\
        Optimizer & Adam \\
        DQN learning rate & $10^{-3}$ \\
        PPO learning rate & $3\times 10^{-4}$ \\
        Batch size & 64 \\
        Buffer capacity (DQN)& $10^5$ \\
        Rollout length (PPO) & $2048$ \\
        \midrule
        \multicolumn{2}{l}{\textit{DQN}} \\
        Priority exponent~\cite{PER} $\alpha_{\text{PER}}$ & $0.6$ \\
        IS-weight exponent~\cite{PER} $\beta_{\text{PER}}$ & $0.4$, annealed
        to $1$ (increment $10^{-4}$ per update) \\
        \midrule
        \multicolumn{2}{l}{\textit{Sparsified DeepRL variants}} \\
        Sparsified networks (PPO variants) & critic only \\
        Sparsity level &
        $\{0.9, 0.99, 0.999\}$ \\
        Pruning schedule (dense-to-sparse)~\cite{dense2sparse} &
        cubic, global magnitude, $20\%$--$80\%$ of training,
        every $500$ steps \\
        Drop fraction, cosine decay~\cite{cosine}
        (SET~\cite{mocanu17scalable}, RiGL~\cite{rigl}) &
        $0.5$, annealed to $0$ at $80\%$ of training,
        every $500$ steps \\
        Sparsity distribution (SET~\cite{mocanu17scalable},
        RiGL~\cite{rigl}) & Erd\H{o}s-R\'{e}nyi kernel (ERK) \\
        \bottomrule
    \end{tabular}
\end{table}

\singlespacing
\setlength{\bibsep}{0pt}
\footnotesize
\printbibliography

\end{document}

%% file: MinhPreamble.tex
\usepackage[T1]{fontenc}

\usepackage{calrsfs}
\usepackage{amsmath, amsthm, accents, amssymb, amsfonts, bm,
  mathtools, cases}
\usepackage[cal = boondox,%
  calscaled = 1.1,%
  bb = ams,%
  bbscaled = .94,%
  frakscaled = .97,%
  scr = esstix]{mathalpha}

\usepackage[linesnumbered,vlined,ruled,commentsnumbered]{algorithm2e}
\SetKwInput{Require}{Require}
\SetAlCapHSkip{0pt}                

\usepackage[inline]{enumitem}
\usepackage{graphicx}
\usepackage{wrapfig}

\usepackage{caption}
\usepackage{subcaption}
\usepackage{hyperref}
\hypersetup{hidelinks}

\usepackage[capitalize]{cleveref}
\usepackage{url}
\usepackage{booktabs}
\usepackage{array}
\usepackage{multirow}

\usepackage{textcomp}

\usepackage{etoolbox}
\def\BibTeX{{\rm B\kern-.05em{\sc i\kern-.025em b}\kern-.08em
    T\kern-.1667em\lower.7ex\hbox{E}\kern-.125emX}}

\usepackage{siunitx}
\usepackage[mathlines, switch]{lineno}

\usepackage[textsize=footnotesize]{todonotes}

\usepackage{setspace}

\newcommand\given{{\mathbin{}\mid\mathbin{}}}
\newcommand{\IntegerP}{\mathbb{N}}
\newcommand{\IntegerPP}{\mathbb{N}_*}
\newcommand{\Real}{\mathbb{R}}

\newcommand{\PD}{\mathbb{S}_{++}}
\newcommand\SetSymbol[1][]{
  \nonscript\,#1\vert \allowbreak \nonscript\,\mathopen{}}
\DeclarePairedDelimiterX\Set[1]{\lbrace}{\rbrace}%
{ \renewcommand\given{\SetSymbol[\delimsize]} #1 }
\DeclarePairedDelimiterX\norm[1]\lVert\rVert{\ifblank{#1}{\:\cdot\:}{#1}}
\DeclarePairedDelimiterX\innerp[2]{\langle}{\rangle}{#1
  \mathop{}\delimsize\vert\mathop{} #2}

\newcommand\vect[1]{\mathbf{#1}}
\newcommand\vectgr[1]{\bm{#1}}

\newcommand{\Banach}{\mathcal{B}}
\newcommand{\Bellman}{\mathbfcal{T}}

\DeclareMathOperator{\Fix}{Fix}

\DeclareMathOperator*{\Argmin}{arg\,min}

\DeclareMathOperator*{\Argmax}{arg\,max}

\DeclareMathOperator{\Df}{D}
\DeclareMathOperator{\trace}{tr}

\DeclareMathOperator{\Exp}{Exp}
\DeclareMathOperator{\Log}{Log}

\DeclareMathOperator{\grad}{grad}
\DeclareMathOperator{\rank}{rank}
\DeclareMathOperator{\inv}{inv}

\newtheorem{theorem}{Theorem}

\newtheorem{proposition}[theorem]{Proposition}

\crefname{theorem}{Theorem}{Theorems}
\crefname{lemma}{Lemma}{Lemmata}
\crefname{proposition}{Proposition}{Propositions}
\crefname{assumption}{Assumption}{Assumptions}
\crefname{assumptions}{Assumptions}{Assumptions}
\crefname{algo}{Algorithm}{Algorithms}
\crefname{fact}{Fact}{Facts}
\crefname{claim}{Claim}{Claims}
\crefname{corrolary}{Corollary}{Corollaries}

\newlist{thmlist}{enumerate}{1}
\setlist[thmlist]{label=\textbf{(\roman{*})}, ref=\thetheorem(\roman{*}), noitemsep}

\newlist{lemlist}{enumerate}{1}
\setlist[lemlist]{label=\textbf{(\roman{*})}, ref=\thelemma(\roman{*}), noitemsep}

\newlist{exlist}{enumerate}{1}
\setlist[exlist]{label=\textbf{(\roman{*})}, ref=\theexample(\roman{*}), noitemsep}

\newlist{factlist}{enumerate}{1}
\setlist[factlist]{label=\textbf{(\roman{*})}, ref=\thefact(\roman{*}), noitemsep}

\newlist{proplist}{enumerate}{1}
\setlist[proplist]{label=\textbf{(\roman{*})},
ref=\theproposition(\roman{*}), noitemsep}

\newlist{asslist}{enumerate}{1}
\setlist[asslist]{label=\textbf{(\roman{*})},
ref=\theassumption(\roman{*}), noitemsep}

\newlist{assslist}{enumerate}{1}  
\setlist[assslist]{label=\textbf{(\roman{*})},
ref=\theassumptions(\roman{*}), noitemsep}

\newlist{deflist}{enumerate}{1}  
\setlist[deflist]{label=\textbf{(\roman{*})},
ref=\thedefinition(\roman{*}), noitemsep}

\newlist{algolist}{enumerate}{1}
\setlist[algolist]{label=\textbf{(\roman{*})}, ref=\thealgo(\roman{*}), noitemsep}

\newlist{claimlist}{enumerate}{1}
\setlist[claimlist]{label=\textbf{(\roman{*})},
ref=\theclaim(\roman{*}), noitemsep}

\crefname{thmlisti}{Theorem}{Theorems}
\crefname{lemlisti}{Lemma}{Lemmata}
\crefname{proplisti}{Proposition}{Propositions}
\crefname{asslisti}{Assumption}{Assumptions}
\crefname{assslisti}{Assumption}{Assumptions}
\crefname{deflisti}{Definition}{Definitions}
\crefname{exlisti}{Example}{Examples}
\crefname{algolisti}{Algorithm}{Algorithms}
\crefname{factlisti}{Fact}{Facts}
\crefname{claimlisti}{Claim}{Claims}

\crefname{figure}{Figure}{Figures}
\crefname{section}{Section}{Sections}

\allowdisplaybreaks

\biboptions{sort&compress}

\makeatletter
\let\c@author\relax
\makeatother
\usepackage[%
  backend = biber,%
  maxbibnames = 10,%
  minbibnames = 2,%
  style = ieee,%
  citestyle = numeric-comp,%
  bibencoding = utf8,%
  defernumbers = true,%
  sortcites=true,%
  natbib=true,%
  doi=true,%
  isbn=false,%
  url=false,%
  eprint=true%
]{biblatex}

%% file: TikzPreamble.tex
\usepackage{tikz}
\usepackage{xcolor}
\usepackage{pgfplots}
\pgfplotsset{compat=newest}
\usepgflibrary{fillbetween}
\usetikzlibrary{angles, arrows, arrows.meta, shapes,
tikzmark, patterns, shapes.geometric}

\pgfdeclareplotmark{halftriangle}{
  \newdimen\dividerwidth
  \dividerwidth=0.6\pgflinewidth
  \newdimen\halfdivider
  \halfdivider=.2\dividerwidth
  \begin{pgfscope}
    \pgfpathmoveto{\pgfpoint{0pt}{\pgfplotmarksize}}
    \pgfpathlineto{\pgfpoint{-\pgfplotmarksize}{-\pgfplotmarksize}}
    \pgfpathlineto{\pgfpoint{\pgfplotmarksize}{-\pgfplotmarksize}}
    \pgfpathclose
    \pgfusepath{clip}
    \pgfpathmoveto{\pgfpoint{-\halfdivider}{\pgfplotmarksize}}
    \pgfpathlineto{\pgfpoint{-\pgfplotmarksize}{-\pgfplotmarksize}}
    \pgfpathlineto{\pgfpoint{-\halfdivider}{-\pgfplotmarksize}}
    \pgfpathclose
    \pgfusepath{fill}
    \pgfsetfillcolor{white}
    \pgfpathmoveto{\pgfpoint{\halfdivider}{\pgfplotmarksize}}
    \pgfpathlineto{\pgfpoint{\pgfplotmarksize}{-\pgfplotmarksize}}
    \pgfpathlineto{\pgfpoint{\halfdivider}{-\pgfplotmarksize}}
    \pgfpathclose
    \pgfusepath{fill}
  \end{pgfscope}
  \pgfpathmoveto{\pgfpoint{0pt}{\pgfplotmarksize}}
  \pgfpathlineto{\pgfpoint{-\pgfplotmarksize}{-\pgfplotmarksize}}
  \pgfpathlineto{\pgfpoint{\pgfplotmarksize}{-\pgfplotmarksize}}
  \pgfpathclose
  \pgfusepath{stroke}
  \begin{pgfscope}
    \pgfsetlinewidth{\dividerwidth}
    \pgfpathmoveto{\pgfpoint{0pt}{\pgfplotmarksize}}
    \pgfpathlineto{\pgfpoint{0pt}{-\pgfplotmarksize}}
    \pgfusepath{stroke}
  \end{pgfscope}
}

\pgfdeclareplotmark{quartedcircle}{
  \begin{pgfscope}
    \pgfpathcircle{\pgfpointorigin}{\pgfplotmarksize}
    \pgfusepath{clip}
    \pgfpathcircle{\pgfpointorigin}{\pgfplotmarksize}
    \pgfusepath{fill}
    \pgfsetfillcolor{white}
    \pgfpathmoveto{\pgfpointorigin}
    \pgfpathlineto{\pgfpoint{\pgfplotmarksize}{0pt}}
    \pgfpathlineto{\pgfpoint{\pgfplotmarksize}{\pgfplotmarksize}}
    \pgfpathlineto{\pgfpoint{0pt}{\pgfplotmarksize}}
    \pgfpathclose
    \pgfusepath{fill}
    \pgfpathmoveto{\pgfpointorigin}
    \pgfpathlineto{\pgfpoint{-\pgfplotmarksize}{0pt}}
    \pgfpathlineto{\pgfpoint{-\pgfplotmarksize}{-\pgfplotmarksize}}
    \pgfpathlineto{\pgfpoint{0pt}{-\pgfplotmarksize}}
    \pgfpathclose
    \pgfusepath{fill}
  \end{pgfscope}
  \pgfpathcircle{\pgfpointorigin}{\pgfplotmarksize}
  \pgfusepath{stroke}
  \pgfpathmoveto{\pgfpoint{-\pgfplotmarksize}{0pt}}
  \pgfpathlineto{\pgfpoint{\pgfplotmarksize}{0pt}}
  \pgfpathmoveto{\pgfpoint{0pt}{-\pgfplotmarksize}}
  \pgfpathlineto{\pgfpoint{0pt}{\pgfplotmarksize}}
  \pgfusepath{stroke}
}

\colorlet{dqn}{green!50!black}
\newcommand{\dqn}{
    \tikz[baseline=-0.5ex]{
        \draw[line width=1pt, color=dqn] (-0.2, 0) -- (0.2, 0);
        \draw[color=dqn] plot[mark size=3pt, mark=square*,
        mark options={line width=1pt}]
        coordinates {(0,0)};
    }
}
\colorlet{duelddqn}{teal!50!black}

\colorlet{ppo}{violet!80!black}
\newcommand{\ppo}{
    \tikz[baseline=-0.5ex]{
        \draw[line width=1pt, color=ppo] (-0.2, 0) -- (0.2, 0);
        \draw[color=ppo] plot[mark size=3pt, mark=triangle*,
        mark options={line width=1pt}]
        coordinates {(0,0)};
    }
}

\colorlet{proposed}{blue!80!black}
\newcommand{\proposed}{
    \tikz[baseline=-0.5ex]{
        \draw[line width=1pt, color=proposed] (-0.2, 0) -- (0.2, 0);
        \draw[color=proposed] plot[mark size=3pt, mark=*,
        mark options={line width=1pt}]
        coordinates {(0,0)};
    }
}

\colorlet{klspi}{orange!80!black}

\colorlet{obr}{magenta!80!black}

\colorlet{emgmm}{red!80!black}

\newcommand{\quicksymbol}[2]{
    \tikz[baseline=-0.5ex]{
        \draw[line width=1pt, color=#1] (-0.2, 0) -- (0.2, 0);
        \draw[color=#1] plot[#2]
        coordinates {(0,0)};
    }
}
\newcommand{\quickline}[2]{
    \tikz[baseline=-0.5ex]{
        \draw[line width=#2, color=#1] (-0.2,0) -- (0.2,0);
    }
}

\newcommand{\quickstyledline}[3]{
    \tikz[baseline=-0.5ex]{
        \draw[line width=#2, color=#1, #3]
        (-0.2,0)--(0.2,0);
    }
}

\newcommand{\dotsymbol}[2]{
    \tikz[baseline=-0.5ex]{
        \draw[color=#1] plot[#2] coordinates {(0,0)};
    }
}
